\documentclass{article} % For LaTeX2e
\usepackage{iclr2027_conference,times}

\usepackage{amsmath,amsfonts,bm}

\def\eqref#1{equation~\ref{#1}}
\def\1{\bm{1}}

\DeclareMathAlphabet{\mathsfit}{\encodingdefault}{\sfdefault}{m}{sl}
\SetMathAlphabet{\mathsfit}{bold}{\encodingdefault}{\sfdefault}{bx}{n}

\usepackage{mathtools}
\usepackage{amsthm}
\usepackage{algorithm} 
\usepackage{algpseudocode}
\usepackage{hyperref}
\usepackage{url}
\usepackage{xspace}
\usepackage{booktabs}
\usepackage{multirow}
\usepackage{graphicx}
\usepackage{xcolor}
\usepackage{wrapfig}
\usepackage{needspace}
\usepackage{capt-of}
\usepackage{subcaption}
\usepackage{enumitem}

\newtheorem{proposition}{Proposition}
\newcommand{\modelname}{S$^3$R\xspace}

\title{\modelname: Fast and Faithful RL Rollouts via Self-Sparse Speculative Decoding}

\author{Yuhang He, Hongyu Ge, Zihan Lin, Siyi Liu, Yongqi Zhang \\
 The Hong Kong University of Science and Technology (Guangzhou)\\
\texttt{yhe534@connect.hkust-gz.edu.cn} \\
}

\iclrfinalcopy % Uncomment for camera-ready version, but NOT for submission.
\begin{document}

\maketitle

\begin{abstract}
Fast rollout generation for long-context reinforcement learning with verifiable rewards (RLVR) requires reducing KV-cache traffic without changing the response groups used for policy optimization. 
In our profiled RLVR workloads, attention accounts for 72--76\% of autoregressive decoding kernel time.
Sparse attention reduces these reads, but sampling directly from a sparse policy changes the response groups used for learning. 
For group-relative optimization, we show that even exact response-wise reweighting does not generally recover the expected dense-policy update because the advantage of each response depends on the rewards of other group members. 
We propose Self-Sparse Speculative Rollouts (\modelname).
It uses sparse attention to propose tokens, then generates responses through full-attention verification and exact speculative sampling.
Drafting and verification share the same policy parameter, eliminating the need to train or update a separate drafter. The resulting response groups follow the dense-policy distribution.
Under \modelname, we characterize speedup conditions and the costs of drafting and verification. 
We found that dense verification still requires the complete historical KV cache. 
A KV-aware scheduler therefore limits KV-cache repeated offloading and restoration while maintaining rollout parallelism. % OR thrashing
In a matched ablation with an aggressive sparse drafting budget, \modelname maintains stable training while direct sparse sampling with training-side correction becomes unstable. 
On long-context mathematical reasoning, \modelname achieves $1.39\times~1.35\times$ the end-to-end training throughput of autoregressive rollouts on Qwen3 model, with comparable task performance. 
% Our code is avaliable at \url{https://anonymous.4open.science/r/S3R-7666}
\end{abstract}
\section{Introduction}
\label{sec:introduction}

Rollout generation is a primary bottleneck in reinforcement learning with verifiable rewards (RLVR) for large language models
(LLMs)~\citep{DBLP:journals/corr/abs-2412-16720,DBLP:journals/corr/abs-2501-12948}.
Algorithms such as Group Relative Policy Optimization (GRPO) samples groups of responses and  compute group-relative advantages.~\citep{shao2024deepseekmath}.
Generating long responses through autoregressive (AR) decoding makes rollouts a substantial component of the overall training pipeline~\citep{DBLP:journals/corr/abs-2606-18967,DBLP:journals/corr/abs-2512-04752,DBLP:journals/corr/abs-2509-18521},
motivating methods of reducing rollout cost for efficient RLVR.

In our profiled workloads, attention accounts for much of this cost:
each decoding step read the KV cache associated with an increasingly long context
(Figure~\ref{fig:sys-profiling-attention}).
Training-free KV compression and sparse attention can reduce these reads by
restricting context~\citep{DBLP:conf/nips/LiHYVLYCLC24,
DBLP:conf/icml/TangZZXKH24,DBLP:journals/corr/abs-2505-24133}.
When used directly for sampling, sparse attention alters the behavior policy
and can destabilize training~\citep{DBLP:conf/acl/LuoZHZWSSLZ26,DBLP:journals/corr/abs-2606-08446} (Figure~\ref{fig:intro-sparse-stability}).
Training-side filtering and importance reweighting can mitigate this
mismatch~\citep{DBLP:conf/acl/LuoZHZWSSLZ26}.
In GRPO, however, each response's advantage depends on the rewards of
other group members.
We show that even exact response-wise reweighting need not recover the
expected dense policy update.
We therefore seek  to accelerate
rollout generation through sparse computation while preserving the dense policy’s response-group distribution.

Speculative decoding (SD) preserves the target distribution through exact sampling, using a cheaper drafter to propose candidate tokens and the target (i.e., the policy itself) to verify them in parallel
~\citep{pmlr-v202-leviathan23a,DBLP:journals/corr/abs-2302-01318,DBLP:conf/iclr/ZhouLRMRKKA24}.
In practice, the speedup from SD depends on the availability of a suitable drafter with sufficient acceptance.
Such a drafter may not be readily available for fine-tuned target models and may require additional
training and deployment effort.
Moreover, repeated policy updates in RLVR can cause a fixed drafter to become stale, further reducing the acceptance rate.
Auxiliary proposers may produce accurate proposals, but maintaining alignment
between the proposer and the evolving policy introduces optimization and systems overhead
~\citep{DBLP:journals/corr/abs-2509-21792,
DBLP:journals/corr/abs-2510-26475,DBLP:conf/asplos/0004YGY0GC0K026}.
History retrieval methods avoid an auxiliary model but rely on past trajectories covering continuations under the evolving policy~\citep{DBLP:journals/corr/abs-2509-23232,
DBLP:journals/corr/abs-2511-13841,DBLP:journals/corr/abs-2508-18588}.
For both approach, decoding progress must amortize drafting and
verification costs across changing workloads~\citep{DBLP:journals/corr/abs-2606-18967}.
These considerations motivate using sparsification of the behavior policy to generate low-cost
proposals, while retaining dense verification
to sample faithfully from the target distribution.

To this end, we present \textbf{S}elf-\textbf{S}parse
\textbf{S}peculative \textbf{R}ollouts (\modelname), a framework
for fast and faithful RLVR rollouts.
We first analyze the mismatch induced
by sparsity at the response-group level and show why exact response-level
reweighting does not necessarily recover the expected group update under the dense policy.
Based on this analysis, \modelname confines sparsity to proposal
generation and combines full-attention verification with exact
speculative sampling to preserve the dense response-group distribution.
Drafting and verification share the same policy weights,
keeping the drafter synchronized without
separate drafter fine-tuning.
Next, to accommodate the changing RLVR workloads, we characterize how \modelname amortizes drafting and verification costs in the runtime.
During rollout generation, dense verification retains the complete historical KV cache.
Under GPU memory pressure, repeated KV offloading and restoration can cause cache thrashing.
We introduce KV-aware scheduling to limit this thrashing while sustaining
rollout parallelism (Figure~\ref{fig:sys-profiling-kv-work}).
Cross-step state management supports policy synchronization and
inference restoration after each actor update.
Finally, experiments on long-context mathematical reasoning demonstrate stable training under aggressive sparse
drafting and higher end-to-end training throughput than AR, with comparable task performance.
Our contributions are threefold:
\begin{itemize}[leftmargin=*, nosep]
    \item \textbf{Group-level mismatch analysis.}
    We show that exact response-wise reweighting need not recover the
    dense policy's expected update, and provide formulation and design principles for RL rollout.

\item \textbf{Faithful sparse self-speculative rollouts.}
We develop \modelname by integrating sparse self-drafting with
exact sampling, preserving the dense response-group distribution
without separate drafter training.
We characterize its speedup conditions and provide runtime support
for efficient RLVR.

    \item \textbf{Practical long-context acceleration.}
On mathematical RLVR workloads, \modelname increases end-to-end
training throughput over AR by 1.39$\times$ on Qwen3-1.7B and 1.35$\times$ on
Qwen3-8B, while maintaining stable training and comparable
task performance.
\end{itemize}

% Sparse drafting and full-attention verification share the same
% policy weights, eliminating separate drafter training and updates.
% Verification supplies dense probabilities for exact sampling and attention
% statistics to refresh the sparse context for the next draft block.
% At fixed learner parameters, unchanged training operations therefore yield 
% gradient estimates with the same distribution as those obtained from dense AR rollouts.
% Full-attention verification still requires complete dense KV state.
% As responses grow, the active KV working set can exceed GPU capacity,
% triggering offloading and restoration.
% We therefore use a KV-aware AIMD scheduler to balance rollout parallelism
% against repeated KV reconstruction and transfer
% (Figure~\ref{fig:sys-profiling-kv-work}).
% Cross-step state management synchronizes policy weights and restores
% rollout execution after each actor update.
% Experiments on long-context mathematical reasoning show stable training
% with a 5\% drafting budget, whereas the evaluated direct sparse variant
% with training-side corrections collapses at the same budget
% (Figure~\ref{fig:stability-verification}).
% \modelname sustains end-to-end speedups as the policy and response lengths
% evolve (Figure~\ref{fig:rollout_acceleration_discussion}).
% It achieves $1.39\times$ end-to-end training throughput on Qwen3-1.7B and
% $1.35\times$ on Qwen3-8B relative to AR rollouts, with
% comparable task performance.

\begin{figure}[t]
  \vspace{-5pt}
  \centering
  \captionsetup[subfigure]{font=small,skip=2pt}
  \begin{subfigure}[b]{0.32\linewidth}
    \centering
    \includegraphics[width=\linewidth]{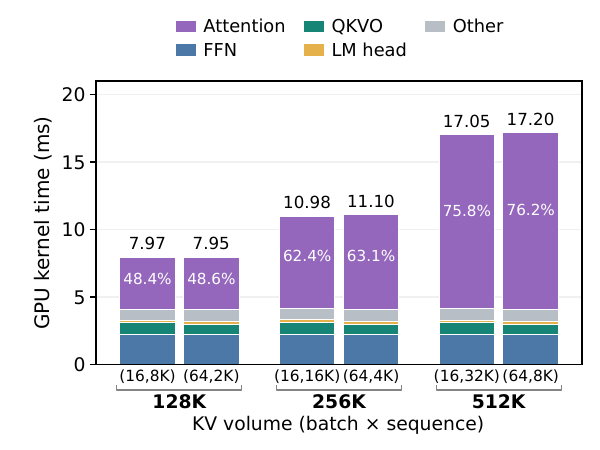}
    \caption{Attention dominates AR.}\label{fig:sys-profiling-attention}
  \end{subfigure}\hfill
  \begin{subfigure}[b]{0.32\linewidth}
    \centering
    \includegraphics[width=\linewidth]{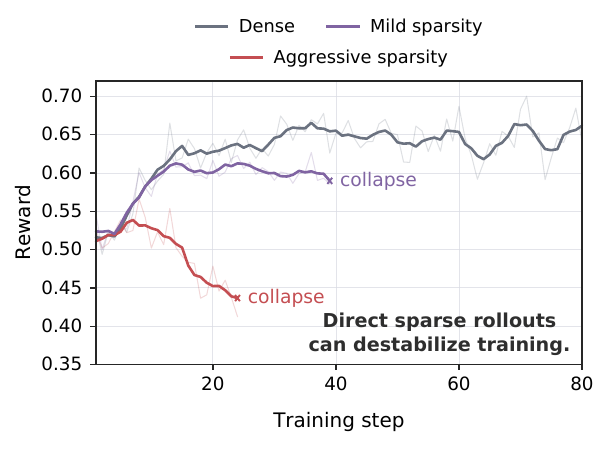}
    \caption{Sparse-rollout collapse.}\label{fig:intro-sparse-stability}
  \end{subfigure}\hfill
  \begin{subfigure}[b]{0.32\linewidth}
    \centering
    \includegraphics[width=\linewidth]{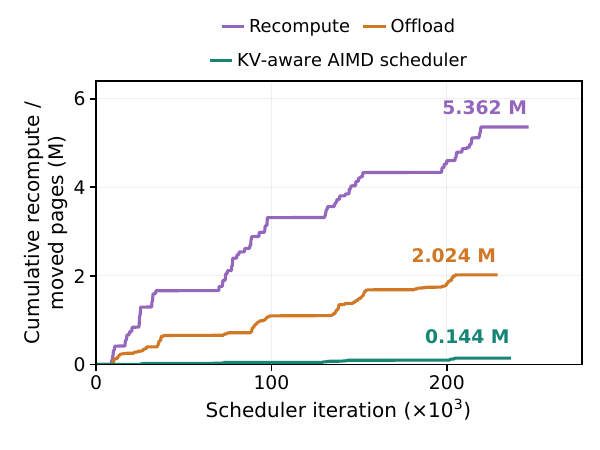}
    \caption{KV-aware scheduling.}\label{fig:sys-profiling-kv-work}
  \end{subfigure}
  \caption{
  (\subref{fig:sys-profiling-attention}) Attention accounts for 72--76\% of AR-decoding kernel time as KV volume grows (Full results in Appendix~\ref{app:why-ssrl-e1}).
  (\subref{fig:intro-sparse-stability}) Both mild and aggressive sparsity exhibit training collapse without training-side mismatch correction.
  (\subref{fig:sys-profiling-kv-work}) Cumulative KV pages recomputed or transferred; our scheduler substantially reduces repeated KV reconstruction and transfer.
\vspace{-18pt}
  }
  \label{fig:sys-profiling}
\end{figure}

\section{Related Work}
\label{sec:related_works}

\subsection{Efficient Reinforcement Learning for LLMs}
% 参考 docs/reference_paper/efficient_rollout/latex/paper.tex 的 2.1 Reinforcement Learning for LLMs
% 同时参考 docs/reference_paper/sparse_rl/latex/paper.tex 的 2 Efficient Reinforcement Learning for LLM
% 注意这里着重分析 RL 的瓶颈和过程，不与下面的 subsection重叠
RL post-training methods improve the efficiency and stability of policy
optimization~\citep{DBLP:journals/corr/SchulmanWDRK17,shao2024deepseekmath,
DBLP:journals/corr/abs-2501-12948,DBLP:journals/corr/abs-2507-18071,
DBLP:journals/corr/abs-2512-23075}.
System work further improves GPU utilization through flexible orchestration,
asynchronous overlap, and tail-aware scheduling~\citep{DBLP:conf/eurosys/ShengZYWZZPL025,
DBLP:conf/nips/FuGSZMHXWMWYYW25,DBLP:journals/corr/abs-2511-14617,
DBLP:journals/corr/abs-2603-23414}.
Nevertheless, rollout generation remains sequential across token positions,
and long responses increase both decoding latency and KV-cache
pressure~\citep{DBLP:journals/corr/abs-2606-18967,DBLP:conf/sosp/KwonLZ0ZY0ZS23}.
These systems improve the orchestration of RL training and are 
complementary to reducing the cost of sampling each rollout.

\subsection{Sparse and Other Approximate RL Rollouts}
Approximate rollouts (e.g. Sparsification) reduce generation cost but sample from a policy that
differs from the dense policy rollout. Sparse-attention approaches
alleviate this gap through filtering, importance weighting, or approximation
schedules constrained by distributional
statistics~\citep{DBLP:conf/acl/LuoZHZWSSLZ26,DBLP:journals/corr/abs-2606-08446}.
Low-precision rollout methods instead use mismatch-aware optimization or align
training and inference numerics~\citep{DBLP:journals/corr/abs-2602-13953,
DBLP:journals/corr/abs-2510-18855,DBLP:journals/corr/abs-2605-13907,
DBLP:conf/acl/GuWLLZLXWYLHG26,DBLP:journals/corr/abs-2510-26788}.
Other approaches reduce generation cost through early-exit inference or proxy
rollout models~\citep{DBLP:conf/acl/ElhoushiSLHWL0A24,DBLP:journals/corr/abs-2602-06107}.
These mechanisms control or compensate for approximation in rollout generation. \modelname
instead uses self-sparse proposals and full-attention verification
before reward and advantage computation. It targets long-context
KV traffic without changing the standard dense
rollout-policy distribution.

\subsection{Speculative Decoding for Inference and RL Rollouts}
Speculative decoding amortizes autoregressive decoding by verifying multiple
draft tokens in parallel and, under exact rejection sampling, preserves the
target distribution~\citep{pmlr-v202-leviathan23a,DBLP:journals/corr/abs-2302-01318}.
Methods mainly differ in their draft source. Retrieval-based
methods reuse stored text or cached rollouts, avoiding drafter training but
depending on recurrence between available proposals and new
trajectories~\citep{DBLP:conf/naacl/0012ZCL024,DBLP:journals/corr/abs-2509-23232,
DBLP:journals/corr/abs-2601-09083,DBLP:journals/corr/abs-2511-13841}.
Auxiliary proposers can be adapted online to changing policies or serving
workloads, at the cost of maintaining and updating a separate
network~\citep{DBLP:journals/corr/abs-2509-21792,
DBLP:journals/corr/abs-2510-26475,DBLP:conf/asplos/0004YGY0GC0K026,
DBLP:journals/corr/abs-2602-06932}.
Target-induced methods instead derive the drafter from the current policy.
They use layer skipping~\citep{DBLP:conf/acl/Zhang00S0CM24,DBLP:conf/iclr/Xia00DL25,
DBLP:conf/acl/ChenSWWLLWAY25,DBLP:conf/eacl/SongXZLXLL26},
sparse attention~\citep{DBLP:journals/corr/abs-2512-01278,DBLP:journals/corr/abs-2602-07223,
DBLP:journals/corr/abs-2607-27735}, or quantization~\citep{DBLP:conf/icml/TiwariXTH0HNMKG25}
to reduce drafting cost.
\modelname uses sparse-attention self-drafting for group-based RLVR:
it analyzes the limitations of response-wise correction, specifies conditions for preserving dense response groups, and supports repeated rollout–optimization transitions.
Concurrent EfficientRollout~\citep{DBLP:journals/corr/abs-2606-18967} uses quantized self-drafting to reduce weight-loading costs. Both approaches avoid separate drafter training, but target different sources of rollout cost.

\section{Problem Formulation and Design Principles}
\label{sec:preliminary}

\subsection{Sparsity-induced Mismatch in RLVR}
\label{sec:preliminary-policy-mismatch}

In RLVR, policy-gradient updates depend on the behavior policy that generates responses to a prompt $x$.
Let $o^{i}$ denote the $i$-th complete response, $o_t^{i}$ its token at
position $t$, and $o_{<t}^{i}$ its preceding prefix.
The learner assigns probability $\pi_\theta(o_t^{i}\mid x,o_{<t}^{i})$
to next token generation. The
behavior policy is typically a frozen dense snapshot
$\pi^{\mathrm{dense}}_{\theta_{\mathrm{old}}}$ of the learner. For each prompt,
$G$ responses sampled independently from this policy form a group
$\mathcal O=(o^{1},\ldots,o^{G})$, whose rewards determine the group-relative
advantages used in policy optimization.

To reduce rollout generation cost, we investigate sparsification of $\pi^{\mathrm{dense}}_{\theta_{\mathrm{old}}}$ as a rollout generation mechanism. Sparse attention restricts KV reads without changing the policy parameters, but changes the resulting token distribution.
Running the frozen snapshot with a sparse budget $s\in(0,1]$ defines a proposal policy
$\pi^{\mathrm{sparse}}_{\theta_{\mathrm{old}},s}$.
If sparse responses are used directly for training,
$\pi^{\mathrm{sparse}}_{\theta_{\mathrm{old}},s}$ becomes the behavior policy.
The corresponding token-level importance sampling (IS) ratio decomposes
as:
\begin{equation}
 \begin{aligned}
 \hat{w}
  =\underbrace{\frac{\pi_\theta(o_t^{i}\mid x,o_{<t}^{i})}
                   {\pi^{\mathrm{dense}}_{\theta_{\mathrm{old}}}(o_t^{i}\mid x,o_{<t}^{i})}}
   _{\mathclap{\scalebox{0.9}{$\scriptstyle w(\theta)\;\text{standard policy update}$}}}
 \;\times\;
 \underbrace{\frac{\pi^{\mathrm{dense}}_{\theta_{\mathrm{old}}}(o_t^{i}\mid x,o_{<t}^{i})}
                   {\pi^{\mathrm{sparse}}_{\theta_{\mathrm{old}},s}(o_t^{i}\mid x,o_{<t}^{i})}}
   _{\mathclap{\scalebox{0.9}{$\scriptstyle w_s\;\text{sparsity-induced mismatch}$}}}.
 \end{aligned}
 \label{eq:preliminary-mismatch-decomposition}
\end{equation}
where $w(\theta)$ captures the standard discrepancy caused by policy updates, and $w_s$ arises from sparsification.
Our objective is to accelerate generation while preserving the faithfulness of the final response groups by $\pi^{\mathrm{dense}}_{\theta_{\mathrm{old}}}$.
Because GRPO advantages depend on the complete group, we next examine whether correcting each response separately is sufficient to recover the dense policy update.

% This preservation target concerns complete response groups, whereas Eq.~\ref{eq:preliminary-mismatch-decomposition} describes a token-level
% discrepancy. We therefore examine whether reweighting sampled responses after generation is sufficient to recover the dense policy's group-relative update.

\subsection{Implications of Sparse Rollouts for GRPO Updates}
\label{sec:preliminary-learning-implications}

Training-side methods mitigate the mismatch in
Eq.~\ref{eq:preliminary-mismatch-decomposition} by reweighting or filtering sampled responses~\citep{DBLP:conf/acl/LuoZHZWSSLZ26,DBLP:journals/corr/abs-2510-18855}. 
To isolate the effect of group-relative advantage computation, we consider an idealized correction that uses exact response likelihood ratios rather than local token ratios (e.g., IS). This asks whether reweighting each response’s own probability is sufficient when its update depends on the other group members.
We denote $P(o^{i}\mid x)=\pi^{\mathrm{dense}}_{\theta_{\mathrm{old}}}(o^{i}\mid x)$
and
$Q_s(o^{i}\mid x)=\pi^{\mathrm{sparse}}_{\theta_{\mathrm{old}},s}(o^{i}\mid x)$
as the probabilities of the complete response $o^{i}$ under
dense and sparse rollout, respectively.

\begin{proposition}[Residual GRPO Update Bias under Response-Wise Reweighting]
\label{prop:ssrl-response-correction} % 我删除了“, and assume support coverage and finite expectations“，我觉得这部分可是是默认，所以多余了
Fix the learner and rollout parameters. Multiplying each sparse response’s otherwise unchanged update contribution by its exact response weight $W_i=P(o^{i}\mid x)/Q_s(o^{i}\mid x)$, does not, in general, recover the expected update under dense response groups.
\end{proposition}

The weight $W_i$ corrects for response $i$'s own sampling mismatch, 
but its GRPO advantage still depends on the rewards of other responses drawn from $Q_s$.
Thus, sparse sampling changes not only which response is observed, but also the reward comparisons that determine its group-relative advantage.
Appendix~\ref{app:ssrl-exactness} provides the supporting assumptions and proofs.
To retain the baseline update without redesigning group-level correction, we place the sparsification in candidate generation and preserve faithfulness of the response groups to the dense policy.

\subsection{Preserving Dense Response Groups via Exact Sampling}
\label{sec:preliminary-self-sd-claim}

This placement requires a sampling procedure whose outputs follow 
taregt distribution, even when its candidates are not. We achieve this separation through exact speculative sampling.
At each candidate position,  write $p(v)$ and $q_s(v)$ for the
target and drafted candidate probabilities at the same prefix. Standard speculative
sampling~\citep{pmlr-v202-leviathan23a,DBLP:journals/corr/abs-2302-01318}
checks candidates in token order, accepting each with probability
$\min\{1,p(v)/q_s(v)\}$ using fresh randomness.
At the first rejection, the rejected candidate and its follow-up suffix are discarded, and a correction token is sampled from the residual distribution.
\begin{equation}
 r_s(v)=\frac{[p(v)-q_s(v)]_+}{\sum_u[p(u)-q_s(u)]_+},
 \label{eq:ssrl-residual-sampling}
\end{equation}
where $[z]_+=\max\{z,0\}$.
Acceptance and residual sampling together recover the dense conditional distribution at 
verified positions.

\begin{proposition}[Dense response-group preservation]
\label{prop:ssrl-exactness}
Fix the rollout snapshot $\theta_{\mathrm{old}}$ and prompt $x$.
Generate all $G$ responses independently with sparse proposals $q_s$
and exact speculative sampling targeting $p$, using the dense baseline's
sampling rules. For any valid sparse budget $s$, the resulting
response-group distribution $\Pr_{\text{\modelname},s}(\mathcal O\mid x)$ matches the
dense AR response-group distribution:
\begin{equation}
 \Pr_{\text{\modelname},s}(\mathcal O\mid x)=\prod_{i=1}^{G}P(o^{i}\mid x),
 \qquad \mathcal O=(o^{1},\ldots,o^{G}).
 \label{eq:ssrl-group-distribution}
\end{equation}
\end{proposition}

At fixed learner parameters, the training pipeline therefore
produces the same distribution as dense AR rollouts.
The learner uses the standard ratio $w(\theta)$ in
Eq.~\ref{eq:preliminary-mismatch-decomposition}, without an additional
sparsity-induced correction.
Appendix~\ref{app:ssrl-exactness}
gives the full conditions and proof.
Under these conditions, the sparse budget changes decoding cost, but not the response-group distribution. It can therefore be selected for 
generation efficiency. We instantiate this separation using sparsification of the $\pi^{\mathrm{dense}}_{\theta_{\mathrm{old}}}$, avoiding a separately trained drafter. Section 4 specifies the drafting and dense verification mechanisms.

\section{\modelname: Self-Sparse Speculative Rollouts for Reinforcement Learning}
\label{sec:methodology}

\modelname implements the separation through two coupled designs.
(1) Self-drafting with sparse KV states to reduce latency without a separately trained proposer. 
(2) Full-attention verification supplies the probabilities for exact sampling and recomputes dense states for subsequent generation. Both designs share a snapshot $(\theta_{\mathrm{old}})$, synchronized after each actor update and held fixed until all requested responses are complete. 
Figure~\ref{fig:ssrl-overview} illustrates these connections, and
Algorithm~\ref{alg:ssrl-rollout} gives the rollout and policy-update loop.
% \modelname organizes sparse
% self-drafting and exact full-attention verification around the RL
% training loop.
% At each RL step, we synchronize the actor’s weights to a shared rollout 
% snapshot $(\theta_{\mathrm{old}})$ for drafting and verification, 
% and keep it fixed until all response groups are complete.
% % Sharing the rollout-policy weights keeps the drafter synchronized with
% % the actor, avoiding separate drafter training or updates.
% Figure~\ref{fig:ssrl-overview} illustrates these connections, and
% Algorithm~\ref{alg:ssrl-rollout} gives the rollout and policy-update loop.

\begin{figure*}[t]
  \centering
  \includegraphics[width=\textwidth]{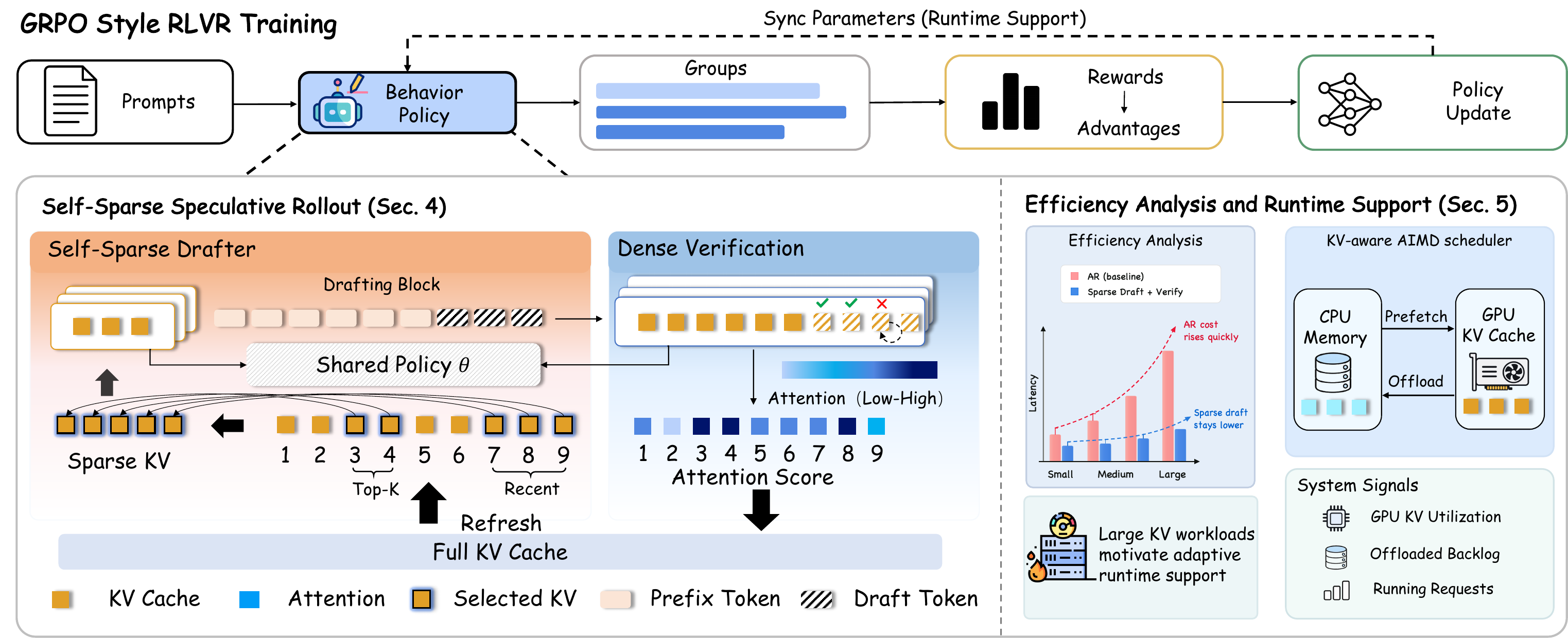}
  \caption{
  Overview of \modelname's rollout mechanism and runtime support.
  Sparse drafting reduces RL rollout cost, while dense verification
  preserves rollout fidelity with KV-aware runtime support.}
  \label{fig:ssrl-overview}
\end{figure*}

\paragraph{Sparse Drafting from the Shared Policy.}
Dense prefill generates the first output token. The drafter uses the full
KV cache for the first block. For each later block, it selects historical KV
entries using attention scores from the previous verification
pass, following \citet{DBLP:journals/corr/abs-2512-01278}.
Let $\mathcal C_\ell$ denote the valid cached KV index set at layer $\ell$
when the sparse context is refreshed. The configured sparse budget ratio
$s$ specifies the target fraction of these indices to retain.
For each KV head (head indices omitted), we retain recent cached positions
$\mathcal R_\ell\subseteq\mathcal C_\ell$ around the verification boundary
and select the remaining positions using the verifier's attention scores
$\mathbf a_\ell$:
\begin{equation}
 \mathcal I_\ell=\mathcal R_\ell\cup
 \operatorname{TopK}_{\lceil s|\mathcal C_\ell|\rceil-|\mathcal R_\ell|}(\mathbf a_\ell).
 \label{eq:ssrl-draft-selection}
\end{equation}
Here $|\mathcal R_\ell|\le\lceil s|\mathcal C_\ell|\rceil$, and Top-K
selects the highest-scoring indices in $\mathcal C_\ell\setminus\mathcal R_\ell$.
% This simplified rule omits the minimum-budget constraint detailed in
% Appendix~\ref{app:ssrl-sparse-context}.
Each draft step applies attention to this view of the complete KV cache.
Across $k$ autoregressive steps, it holds the historical selection fixed
and extends $\mathcal I_\ell$ with provisional KV entries for newly processed
tokens. 
The drafter records each candidate and its full proposal distribution for the exact sampling step.

\paragraph{Full-Attention Verification and Context Refresh.}
The verifier evaluates the draft tokens in a single full-attention causal forward.
Full attention alone is not sufficient if it reuses approximate historical states; verification therefore operates on dense historical KV and recomputes the proposed block through all layers.
Then, we apply the exact sampling in Section~\ref{sec:preliminary-self-sd-claim}.
Verification replaces provisional draft states with recomputed dense KV entries, and removes the unused candidate suffix. 
The verifier's attention weights are averaged to obtain $\mathbf a_\ell$.
After verification, we update $\mathcal C_\ell$ and $\mathcal R_\ell$ for the accepted 
prefix and refresh $\mathcal I_\ell$ using the verifier’s scores in Eq.~\ref{eq:ssrl-draft-selection}.
Once all requested responses are complete, the trainer assembles them by
prompt and group identity and computes rewards and group-relative
advantages.

\section{Efficiency Analysis and Runtime Support}
\label{sec:efficiency-runtime}

Self-sparse drafting reduces historical KV reads but retains
projections and feed-forward computation, and dense verification.
% This retained computation can make self-drafting costlier
% than a pretrained auxiliary proposer but shared weights 
% avoid separate drafter training and parameter staleness.
Its benefit therefore depends on the cost of SD and on whether runtime state management preserves the resulting computational savings.

\paragraph{Latency Analysis.}
\label{sec:efficiency-latency}

Consider a locally fixed active batch size 
$B$ and cached context length $S$. Let
$T_p$ and $T_q$ denote the dense and sparse per-step latencies, and
$T_V$ the verification latency. Each round drafts $k$ tokens and outputs
$\tau$ tokens on average per request (block efficiency), including a correction or bonus token when applicable. For sequential drafting and verification,
the estimated round latency $T_{\mathrm{SD}}$ and speedup $\eta$ over AR are~\citep{pmlr-v202-leviathan23a,sadhukhan2025magicdec}:
\begin{equation}
 T_{\mathrm{SD}}=kT_q+T_V,
 \qquad
 \eta=\frac{T_p}{T_{\mathrm{SD}}/\tau}.
 \label{eq:efficiency-sd-accounting}
\end{equation}
Acceleration therefore requires
$
 \eta>1
 \Longleftrightarrow
 \tau>({kT_q+T_V})/{T_p}.
$
The sparse budget therefore trades off drafting cost against accepted progress: higher acceptance is useful only when its benefit exceeds the additional drafting cost.

% Reserve room for the complete figure; wrap only the two adjacent paragraphs.
\begingroup
\setlength{\intextsep}{4pt}
\Needspace{18\baselineskip}
\begin{wrapfigure}[14]{r}{0.31\linewidth}
  \centering
  \includegraphics[width=\linewidth]{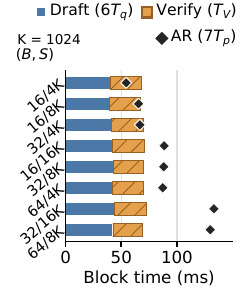}
  \captionsetup{font=footnotesize,skip=2pt,justification=raggedright,singlelinecheck=false}
  \caption{Trade-off costs on Qwen3-8B ($s=0.05$, $k=6$).}
  \label{fig:efficiency-stage-costs}
\end{wrapfigure}
Figure~\ref{fig:efficiency-stage-costs} illustrates this trade-off using separately measured stage latencies with $k=6$.
At $(B,S)=(16,4\mathrm{K})$, drafting and verification  already exceed the $7T_p$ AR reference, even with full acceptance.
With an aggregate KV volume $N_{\mathrm{KV}}=B\cdot S\approx5.24\times10^5$ tokens, the corresponding break-even block efficiency falls to 
about 3.8 output tokens per round (Table~\ref{tab:why-ssrl-e2-stages}).
These serialized estimates exclude no-forward stalls and cross-request phase mixing; Appendix~\ref{app:why-ssrl-cost} extends the analysis to mixed-stage execution.

% leaving a margin for acceptance loss
% and additional runtime costs.
% These costs give a break-even block efficiency of about
% 3.8 output tokens per round, before runtime overhead.
% The estimates assume sequential stages and exclude intervals without
% forward execution.
% Appendix~\ref{app:why-ssrl-cost} gives the cost breakdown and extends
% the analysis to mixed drafting and verification across requests.
% This cost margin is particularly relevant to long-response RLVR, in which sparse drafting can save attention cost.

\paragraph{Runtime Support for RL Rollouts.}
\label{sec:runtime-support}

Dense verification retains the full historical KV cache even when sparse
drafting reduces reads. Under GPU memory pressure, repeated eviction and
recovery can cause KV-cache thrashing~\citep{chen2026concur}.
CPU offloading avoids prefix recomputation, but repeated KV transfers can
still consume the available latency margin.

A permanently low concurrency limit can reduce thrashing but sacrifices
batching when capacity is available. We use a KV-aware additive increase
and multiplicative decrease (AIMD) scheduler to adapt request admission.
GPU KV utilization measures resident pressure, while the restoration backlog
records outstanding demand from paused requests.
Under low utilization, the concurrency limit increases additively; otherwise,
high utilization or an excessive backlog triggers multiplicative backoff.
Paused requests are restored before new requests are admitted, subject to the
concurrency limit and available KV capacity.

Paused requests retain their token and proposal states.
All requested responses complete under the same frozen rollout snapshot
and retain their original group identities, maintaining the complete-group
requirement of Proposition~\ref{prop:ssrl-exactness}.
Between RL steps, the runtime releases request and KV state for actor optimization, then synchronizes the updated policy and rebuilds inference state before the next rollout. Appendix~\ref{app:ssrl-implementation} provides the scheduling rules and lifecycle details.
\par
\endgroup

\section{Experiments}
\label{sec:experiments}

\subsection{Setup}
\label{sec:setup}

\paragraph{Training Details.}
We initialize policies from Qwen3-1.7B and Qwen3-8B~\citep{DBLP:journals/corr/abs-2505-09388}
and train them on DAPO-MATH-17k~\citep{DBLP:conf/nips/YuZZYZYDFLLLLLL25}.
Optimization follows the DAPO variant of GRPO~\citep{shao2024deepseekmath}. 
Each batch contains 256
prompts with eight responses per prompt, sampled at temperature 1.0.
The maximum response lengths are 16K and 24K tokens for the 1.7B and 8B
models, respectively.
All methods run in
veRL~\citep{DBLP:conf/eurosys/ShengZYWZZPL025}
on a single node with eight H100 GPUs. For \modelname, both models use
a draft length of $k=6$, a sparse budget ratio of $s=0.05$, and the KV-aware AIMD scheduler in
Appendix~\ref{app:ssrl-scheduling}. Full training and system configurations
are provided in Appendix~\ref{app:experimental-setup}.

\paragraph{Baselines.}
We compare \modelname against dense \textbf{AR} decoding and three representative speculative proposal strategies.
\textbf{N-gram Lookup} uses prompt lookup decoding~\citep{saxena2023prompt}
to propose continuations from matching token suffixes within the current
prompt or response.
\textbf{History Retrieval} follows RhymeRL~\citep{DBLP:journals/corr/abs-2508-18588}
and retrieves draft continuations from previous responses to the same prompt.
\textbf{Auxiliary Proposer} uses a pretrained EAGLE3
network~\citep{DBLP:conf/nips/LiWZZ25} with $k=5$, keeping its weights
frozen as the target policy evolves. The main comparison uses no
additional drafter training; online co-training is examined in
Appendix~\ref{app:checkpoint-portability}. Baseline implementation
details are given in Appendix~\ref{app:experimental-baselines}.

\paragraph{Metrics.}
We report rollout generation time, end-to-end training-step time, and
training throughput (tokens/s per GPU), with speedups relative to AR.
Rollout time covers response generation, while step time also includes
reward computation, actor updates, and system overhead. 
Block efficiency $\tau$ is the average number of output tokens per
speculative round, including the bonus token. We evaluate
model quality using AIME26 and ArXiv-Math pass@1 scores from the best
checkpoint. Measurement and reporting details are provided in
Appendix~\ref{app:experimental-metrics}.

\subsection{End-to-End Acceleration with Comparable Task Performance}

\begin{table}[h]
  \centering
  \small
  \caption{Main results on Qwen3-1.7B and Qwen3-8B. Benchmark scores are pass@1 (\%) from the best checkpoint. Timing and throughput cover the full training run after the first warm-up step. Best and second-best results within each model size are shown in bold and underlined, respectively.}
  \label{tab:main_results}
  \resizebox{\textwidth}{!}{%
  \begin{tabular}{llcccccc}
  \toprule
  \multirow{2}{*}{Model} & \multirow{2}{*}{Method}
  & \multicolumn{2}{c}{Benchmarks}
  & \multicolumn{4}{c}{Efficiency} \\
  \cmidrule(lr){3-4} \cmidrule(lr){5-8}
  & & AIME26 & ArXiv-Math
  & $\tau$ & Rollout (min) $\downarrow$ & Step (min) $\downarrow$ & Throughput/Speedup \\
  \midrule
  \multirow{5}{*}{Qwen3-1.7B}
  & AR & \textbf{37.2} & \textbf{9.95} & N/A & 10.5 (base) & 13.9 (base) & 2308.6 / $1.00\times$ \\
  & N-gram Lookup & \underline{37.1} & 9.84 & 1.31 & 12.5 ($+19.1\%$) & 19.0 ($+36.3\%$) & 1671.0 / $0.72\times$ \\
  & History Retrieval & 36.0 & 9.77 & \underline{2.03} & 10.4 ($-1.7\%$) & 13.3 ($-4.2\%$) & 2412.7 / $1.05\times$ \\
  & Auxiliary Proposer & 36.7 & 9.87 & 1.89 & \underline{9.1 ($-13.9\%$)} & \underline{12.4 ($-10.5\%$)} & \underline{2562.0 / $1.11\times$} \\
  & \textbf{Ours} & \underline{37.1} & \underline{9.89} & \textbf{5.59} & \textbf{6.5 (-38.6\%)} & \textbf{9.9 (-28.5\%)} & \textbf{3213.8 / 1.39$\times$} \\

  \midrule
  \multirow{5}{*}{Qwen3-8B}
  & AR & 65.42 & 15.70 & N/A & 36.1 (base) & 50.2 (base) & 964.8 / $1.00\times$ \\
  & N-gram Lookup & 63.96 & 16.10 & 1.29 & 38.8 ($+7.5\%$) & 55.0 ($+9.6\%$) & 876.5 / $0.91\times$ \\
  & History Retrieval & \underline{66.04} & \underline{16.38} & \underline{2.04} & 36.1 ($0.0\%$) & 49.8 ($-0.8\%$) & 959.3 / $0.99\times$ \\
  & Auxiliary Proposer & \textbf{66.25} & \textbf{16.82} & 1.86 & \underline{31.6 ($-12.5\%$)} & \underline{45.6 ($-9.2\%$)} & \underline{1055.5 / $1.09\times$} \\
  & \textbf{Ours} & 65.42 & 16.28 & \textbf{5.47} & \textbf{24.1 (-33.2\%)} & \textbf{37.4 (-25.5\%)} & \textbf{1302.6 / 1.35$\times$} \\
  \bottomrule
  \end{tabular}
  }
  \end{table}

Table~\ref{tab:main_results} shows that \modelname achieves the highest
end-to-end training throughput among the compared methods while maintaining
comparable task performance. Specifically, it reaches $1.39\times$ and
$1.35\times$ the throughput of AR on Qwen3-1.7B and Qwen3-8B, respectively.
In contrast, N-gram Lookup reduces training throughput because its low block
efficiency ($\tau\approx1.3$) limits the amortization of the verification cost.
Although History Retrieval improves block efficiency, it provides little
overall acceleration. We hypothesize that, at a decoding temperature of 1.0, limited
reuse of reasoning continuations constrains both lookup-based methods, while
history maintenance introduces additional overhead. The frozen Auxiliary
Proposer provides modest throughput gains, but its declining acceptance rate
(Figure~\ref{fig:rollout-entropy-acceptance-8b}) suggests a growing mismatch
between the frozen proposer and the evolving policy. Nevertheless, the
Auxiliary Proposer achieves higher throughput than History Retrieval despite
having lower block efficiency, indicating that overall throughput depends on
both block efficiency and proposal-generation and runtime costs. \modelname
maintains high block efficiency ($\tau\approx5.5$) at $s=0.05$ by using
proposals from the current policy, without requiring historical coverage or
separate drafter adaptation. This high block efficiency amortizes the cost of
dense verification, while sparse drafting reduces KV reads in our
attention-dominated workloads. Finally, the reward and response-length
trajectories of \modelname closely match those of AR
(Appendix~\ref{app:training-dynamics}), which is consistent with preserving
the dense response-group distribution used for learning.

\subsection{Stabilizing Training under Aggressive Sparsity}
\label{sec:stability}
% Allow the stability figure to float so its discussion can fill the preceding page.
\begin{figure}[!htbp]
  \centering
  \begin{subfigure}[t]{0.48\linewidth}
    \centering
    \includegraphics[width=\linewidth]{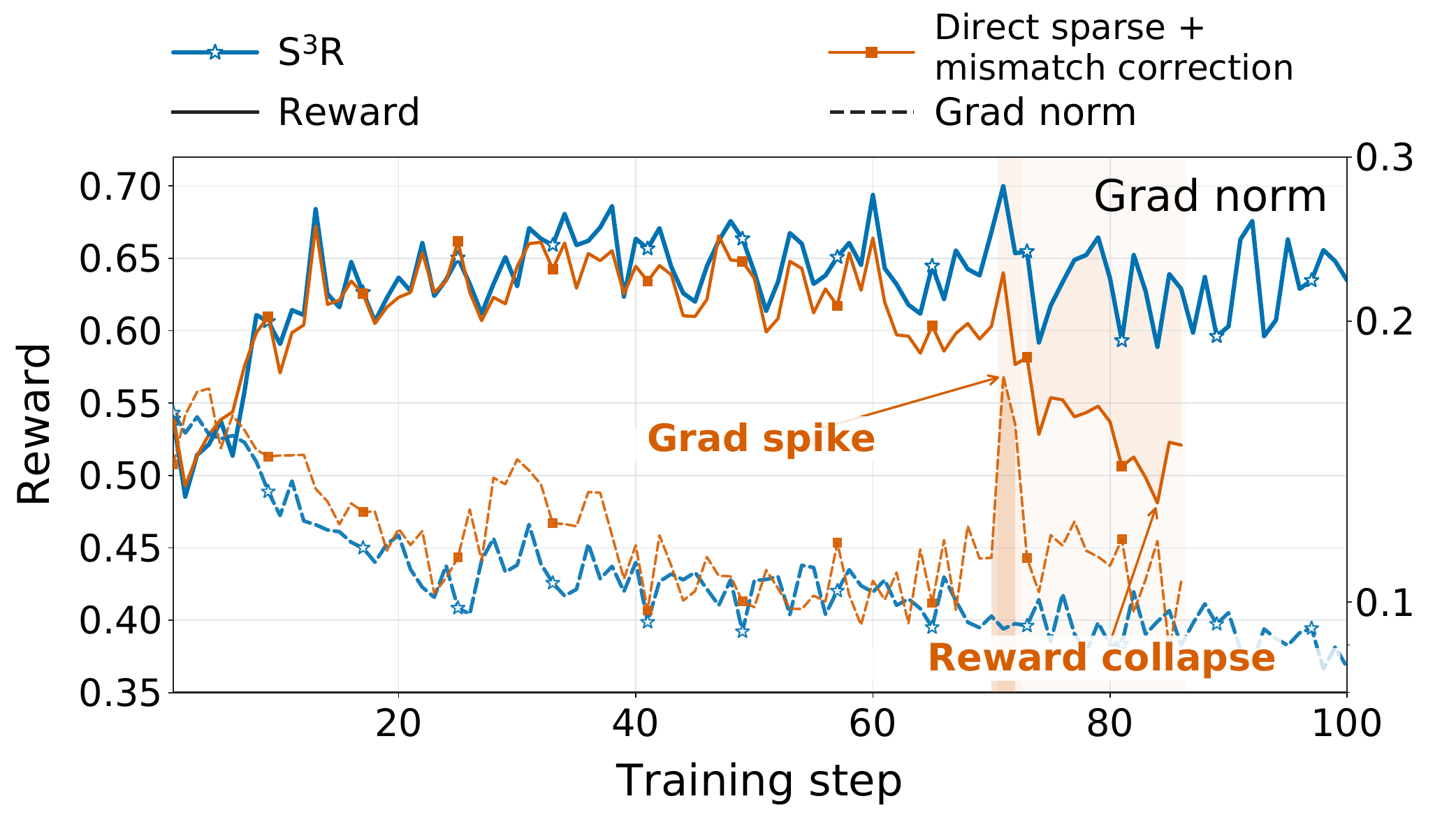}
    \caption{Verification versus mismatch correction.}\label{fig:stability-verification}
  \end{subfigure}\hfill
  \begin{subfigure}[t]{0.48\linewidth}
    \centering
    \includegraphics[width=\linewidth]{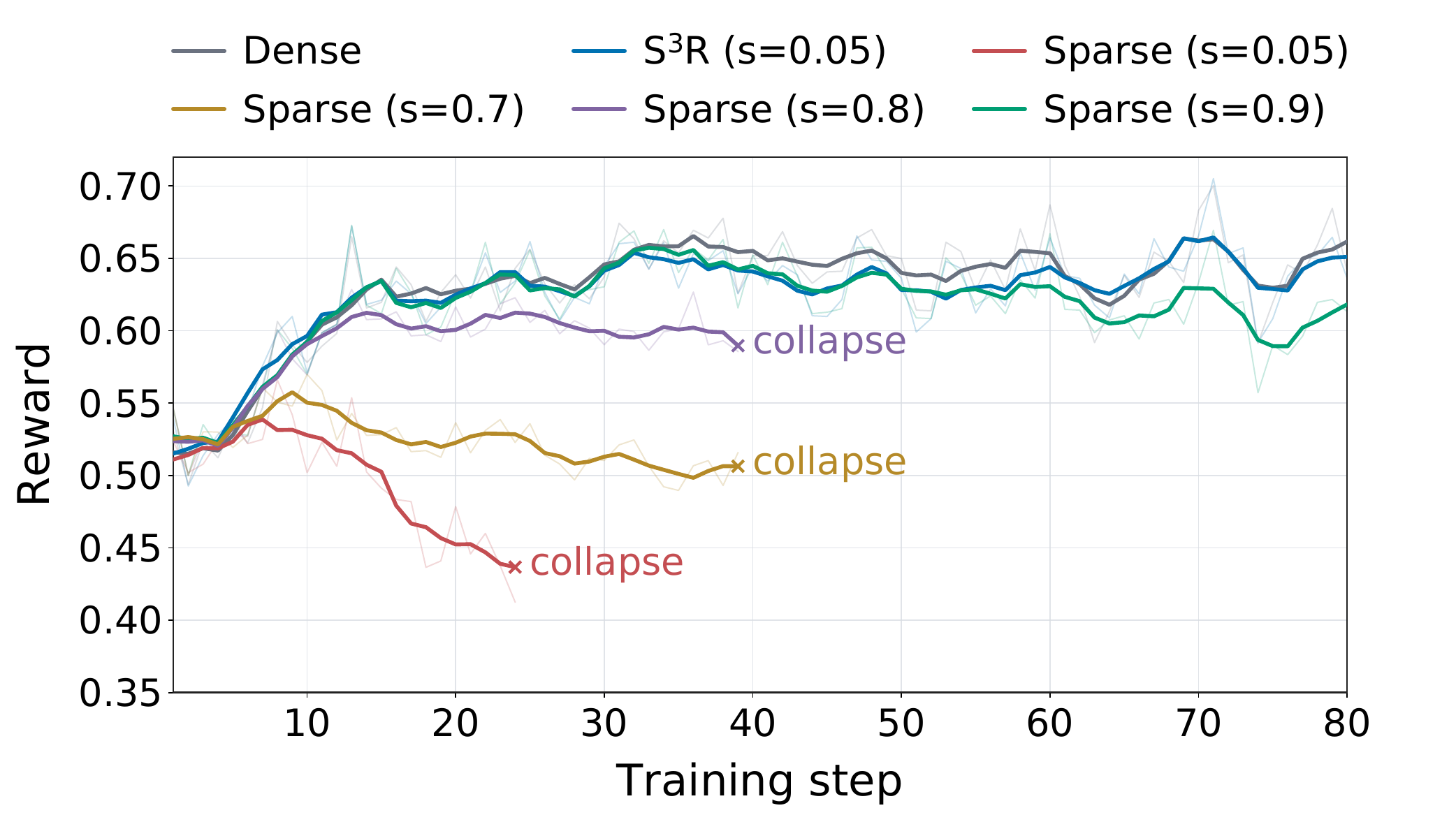}
    \caption{Stability across sparse budgets.}\label{fig:stability-budget}
  \end{subfigure}
  \caption{Training stability under sparse rollout variants.
  (\subref{fig:stability-verification}) Reward and gradient norm for \modelname and Direct sparse + mismatch correction;
  (\subref{fig:stability-budget}) reward across sparse budgets.}
  \label{fig:training_stability_ablation}
\end{figure}

\paragraph{Gradient instability despite training-side correction.}
Figure~\ref{fig:stability-verification} compares \modelname with
\textit{Direct sparse + mismatch correction}, which removes speculative
verification and applies sequence-level filtering and tokenwise importance
reweighting to sparse rollouts (following~\citet{DBLP:conf/acl/LuoZHZWSSLZ26}).
Both use $s=0.05$ with matched model, data, optimizer, and batching settings.
Although their early rewards are similar, the corrected sparse variant
already exhibits larger gradient-norm fluctuations, followed by a sharp
spike around step 70 and further reward deterioration.
\modelname avoids this spike-and-collapse pattern throughout the 100-step run.
Similar mean rewards can therefore mask differences in update stability.
This failure is consistent with the concern in
Section~\ref{sec:preliminary-learning-implications}: correcting a response's
likelihood does not generally restore the reward comparisons within its
GRPO group. The comparison supports addressing this mismatch at sampling
time, before rewards define the group-relative advantages.

\Needspace{4\baselineskip}
\paragraph{Verification enables stable training at aggressive sparsity.}
Figure~\ref{fig:stability-budget} compares Dense and \modelname with direct
sparse rollouts that use neither speculative verification nor training-side mismatch correction.
Direct sparse rollouts collapse early at $s=0.05$ and remain unstable at
$s=0.7$ and $s=0.8$. Only $s=0.9$ sustains relatively stable learning
over the displayed window, indicating that stability constrains how
aggressively these direct rollouts can reduce KV reads.
This sensitivity agrees qualitatively with
\citet{DBLP:journals/corr/abs-2606-08446}, who link stable sparse-rollout
RL to sufficient lower-tail alignment with the dense policy.
% The cited Figure 4 controls fifth-percentile acceptance through dynamic
% budget schedules and uses TIS by default. Its alignment targets are not
% retained-context ratios; the citation supports the qualitative finding.
\modelname follows the Dense reward trajectory at $s=0.05$.
Here, sparsity is confined to proposals, with full-attention verification and
residual sampling preserving the dense rollout distribution.
Together with Proposition~\ref{prop:ssrl-exactness}, these results support
selecting the drafting budget for efficiency without introducing
sparsity-induced changes to the learning signal.

\subsection{Analysis of RL Rollout Efficiency}
\label{sec:rollout-acceleration-analysis}

\begin{figure}[H]
  \centering
  \begin{subfigure}[t]{0.328\linewidth}
    \centering
    \includegraphics[width=\linewidth]{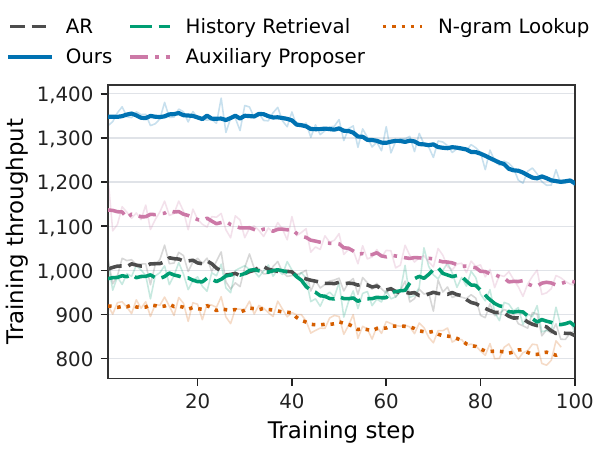}
    \caption{End-to-end training throughput.}\label{fig:rollout-throughput-8b}
  \end{subfigure}\hfill
  \begin{subfigure}[t]{0.328\linewidth}
    \centering
    \includegraphics[width=\linewidth]{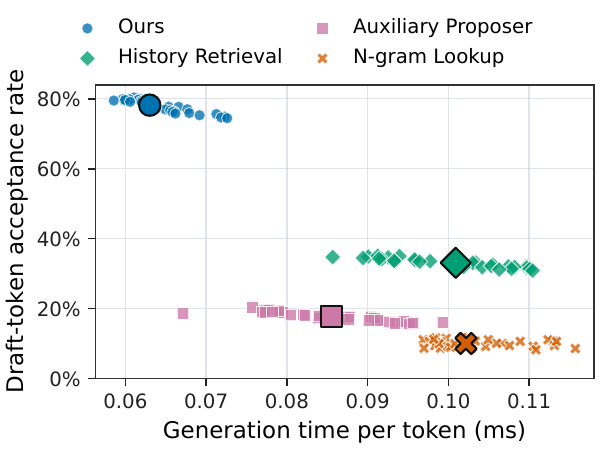}
    \caption{Timing--acceptance trade-off.}\label{fig:rollout-timing-acceptance-8b}
  \end{subfigure}\hfill
  \begin{subfigure}[t]{0.328\linewidth}
    \centering
    \includegraphics[width=\linewidth]{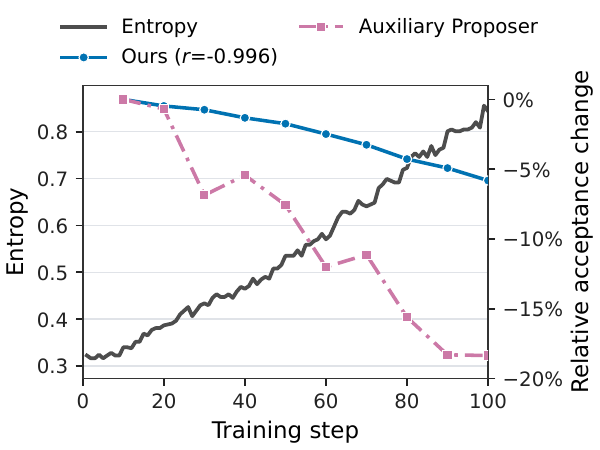}
    \caption{Entropy and acceptance drift.}\label{fig:rollout-entropy-acceptance-8b}
  \end{subfigure}
  \caption{Qwen3-8B rollout diagnostics during RL training.
  (\subref{fig:rollout-throughput-8b}) End-to-end training throughput.
  (\subref{fig:rollout-timing-acceptance-8b}) Generation time per processed
  token (prompt and response) versus draft-token acceptance; outlined
  markers denote method medians.
  (\subref{fig:rollout-entropy-acceptance-8b}) Policy entropy and relative
  change in draft-token acceptance.}
  \label{fig:rollout_acceleration_discussion}
\end{figure}

\paragraph{Sustained efficiency during RL training.}
Figure~\ref{fig:rollout-throughput-8b} shows that \modelname maintains 
higher end-to-end training throughput than the evaluated baselines as
training progresses, including as Qwen3-8B responses grow longer
(Appendix~\ref{app:training-dynamics}).
In Figure~\ref{fig:rollout-timing-acceptance-8b}, \modelname combines 78.2\% median draft-token 
acceptance with the lowest median generation time, 0.0630 ms per processed token, 
among the evaluated speculative methods.
History Retrieval attains higher acceptance than Auxiliary Proposer but
incurs greater generation cost, showing why acceptance alone cannot rank
acceleration methods. 
In long-context RLVR, each accepted block must repay the cost of
drafting, verification, and state management
(Section~\ref{sec:efficiency-runtime}). \modelname benefits from reducing
historical KV reads during drafting in the attention-dominated regime
identified in Figure~\ref{fig:sys-profiling-attention}, while retaining
enough accepted progress to amortize dense verification.
The Qwen3-1.7B results exhibit the same pattern in Appendix~\ref{app:rollout-acceleration-1-7b}.

\paragraph{Proposal alignment as the policy evolves.}
Figure~\ref{fig:rollout-entropy-acceptance-8b} shows a smaller 
relative decline in block efficiency for \modelname than for the 
frozen Auxiliary Proposer, alongside rising policy entropy. 
On Qwen3-8B, \modelname's overall
acceptance decreases relatively 5\%;
later draft positions exhibit larger declines
(Appendix~\ref{app:rollout-position-acceptance}).
Higher entropy reflects greater uncertainty in token selection 
during RL exploration, while ongoing policy updates can leave a frozen drafter stale.
By sharing the current rollout-policy weights, 
\modelname avoids this parameter lag, although sparse approximation can still affect acceptance.
Together with the checkpoint-portability results in
Appendix~\ref{app:checkpoint-portability}, these findings support \modelname's ability to generate
useful proposals across policy updates without
separate drafter training.

% Keep the budget figure with its discussion and end wrapping before the next subsection.
\begingroup
\setlength{\intextsep}{4pt}
\Needspace{16\baselineskip}
\begin{wrapfigure}{r}{0.35\linewidth}
  \centering
  \captionsetup{font=footnotesize,skip=2pt,justification=raggedright,singlelinecheck=false}
  \includegraphics[width=\linewidth]{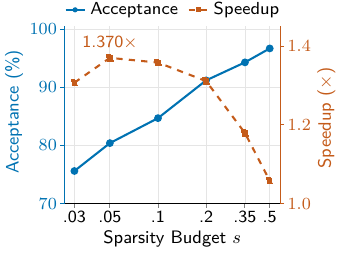}
  \caption{Qwen3-8B budget sweep at a frozen step-60 checkpoint
  ($B=64$, initial context 8K, $k=6$).}
  \label{fig:budget-tradeoff-8b}
\end{wrapfigure}
\paragraph{Balancing draft cost and accepted progress.}
We sweep the sparse drafting budget with full-attention verification enabled, measuring median generation speedup over AR in the same engine (Figure~\ref{fig:budget-tradeoff-8b}). 
Increasing $s$ from 0.05 to 0.50 raises acceptance from 80.4\% to 96.7\%, while speedup falls from 1.370$\times$ to 1.058$\times$.
Retaining more context improves proposal alignment, but the extra KV
reads cost more than the additional accepted tokens save, as anticipated
by Eq.~\ref{eq:efficiency-sd-accounting}.
Reducing $s$ to 0.03 lowers both acceptance and speedup, supporting
our default $s=0.05$ as a practical balance in this workload.
Exact speculative sampling makes this an efficiency trade-off:
the dense response-group distribution is preserved across budgets.
Appendix~\ref{app:why-ssrl-budget} reports sweeps for both model sizes,
and Appendix~\ref{app:why-ssrl-calibration} provides a calibrated cost interpretation.
\par
\WFclear
\endgroup

\subsection{Managing Full-KV State under Memory Pressure}
\label{sec:kv-scheduling-ablation}

\begin{center}
\begin{minipage}{\linewidth}
  \captionof{table}{KV-aware AIMD scheduler ablation on
  Qwen3-1.7B, aggregated over 10 RL steps after warm-up. (a) Generation
  throughput across all methods. (b) KV overhead with fixed concurrency and
  the KV-aware AIMD scheduler, which adaptively adjusts the concurrency limit between 160 and 256.
  All non-AR variants use sparse drafting and full-attention verification.}
  \label{tab:kv_aware_aimd_scheduler_1_7b}
  \centering
  \small
  \setlength{\tabcolsep}{3pt}
  \begin{minipage}[t]{0.47\linewidth}
    \vspace{0pt}
    \centering
    \textbf{(a) Generation throughput}\par\smallskip
    \begin{tabular*}{\linewidth}{@{\extracolsep{\fill}}lrr@{}}
      \toprule
      Method & Tok/s/GPU $\uparrow$ & Gain \\
      \midrule
      AR & 3,989.0 & Ref. \\
      Recompute-256 & 5,894.9 & $+47.8\%$ \\
      Offload-128 & 5,411.8 & $+35.7\%$ \\
      Offload-160 & 5,699.8 & $+42.9\%$ \\
      Offload-256 & 6,014.9 & $+50.8\%$ \\
      \shortstack[l]{\textbf{AIMD scheduler}} & \textbf{6,362.1} & $\mathbf{+59.5\%}$ \\
      \bottomrule
    \end{tabular*}
  \end{minipage}\hfill
  \begin{minipage}[t]{0.50\linewidth}
    \vspace{0pt}
    \centering
    \textbf{(b) KV overhead}\par\smallskip
    \begin{tabular*}{\linewidth}{@{\extracolsep{\fill}}lrr@{}}
      \toprule
      Metric $\downarrow$ & Offload-256 & \shortstack{{AIMD scheduler}} \\
      \midrule
      \shortstack[l]{Peak CPU KV-cache\\occupancy (pages)} & 52,464 & 21,358 \\
      Offload episodes & 1,097 & 59 \\
      Reoffload (\%) & 5.0 & 0.0 \\
      Move/1K (pages) & 18.0 & 1.3 \\
      \bottomrule
    \end{tabular*}
  \end{minipage}

\end{minipage}
\end{center}

Sparse self-drafting reduces KV reads while retaining the complete
historical state required for faithful verification. As long responses
accumulate, this state can exceed GPU capacity, making its management
part of realizing the decoding gain.
In Table~\ref{tab:kv_aware_aimd_scheduler_1_7b}, replacing recomputation
with offloading at a concurrency ceiling of 256 improves generation
throughput by 2.0\%. AIMD adds a further 5.8\% over Offload-256, bringing
the gain over AR to 59.5\%, while reducing transferred pages per 1,000
generated tokens from 18.0 to 1.3.
The throughput improvement alongside reduced KV transfers is consistent
with less KV-cache thrashing as the working set grows.
Offload-128 and Offload-160 both underperform Offload-256: a permanently
conservative limit also sacrifices useful batching when memory is available.
Adapting admission to KV pressure and restoration demand therefore helps
retain rollout parallelism while preserving the full state needed for
dense verification.

\section{Conclusion and Future Work}

We introduced \modelname, which accelerates long-context RLVR through
sparse self-drafting and exact speculative sampling. Under a shared
rollout snapshot, it preserves the dense response-group distribution
without requiring separate drafter training. With KV-aware runtime
support, \modelname achieves $1.39\times$ and $1.35\times$ the end-to-end
training throughput of AR on Qwen3-1.7B and Qwen3-8B, respectively, while
maintaining stable training and comparable task performance. Future
work will explore adaptive sparse budgets and context selection, along
with extensions to MoE and linear-attention models. Another direction
is to coordinate KV management with group completion and
rollout--optimization transitions to reduce overhead while preserving
exact sampling.

\section*{AI use statement}

We used generative AI tools to find related work on sparse RL rollouts and
speculative decoding, and to draft and polish descriptions of \modelname's
sparse self-drafting, dense verification, and runtime support. AI also
assisted with deriving and checking mathematical claims about the limitations
of response-wise importance correction in GRPO and the preservation of dense
response-group distributions under exact speculative sampling, as well as
the latency trade-off between drafting cost and accepted progress. The
authors reviewed all AI-assisted material and take full responsibility for
the paper.

\bibliography{iclr2027_conference}
\bibliographystyle{iclr2027_conference}

\clearpage
\appendix
\section{Policy Mismatch and Distribution Preservation}
\label{app:ssrl-exactness}

This appendix analyzes training-side correction and distribution preservation
under sparse proposals.
Section~\ref{app:ssrl-correction-limits} relates the token-level ratio in
Eq.~(\ref{eq:preliminary-mismatch-decomposition}) to complete-response
importance weights.
Section~\ref{app:ssrl-group-correction} proves
Proposition~\ref{prop:ssrl-response-correction} through a change-of-measure
identity and an explicit GRPO counterexample.
Section~\ref{app:ssrl-sampling-assumptions} proves
Proposition~\ref{prop:ssrl-exactness} from conditional token sampling through
complete responses and independent response groups.

We use the notation in Section~\ref{sec:preliminary-policy-mismatch}
and fix the rollout snapshot $\theta_{\mathrm{old}}$.
For a fixed response index $i$, we write $w_{s,t}$ for the sparsity ratio
$w_s$ evaluated at $(o_t^{i}\mid x,o_{<t}^{i})$.
All policy probabilities include the corresponding sampling transforms.

\subsection{Tokenwise and response-wise importance weights}
\label{app:ssrl-correction-limits}

For a fixed direct-sparse rollout procedure, let
$h_t^i=(x,o_{<t}^i)$ and interpret
$\pi^{\mathrm{sparse}}_{\theta_{\mathrm{old}},s}(\cdot\mid h_t^i)$ as its
prefix-conditional token law. Equation~(\ref{eq:preliminary-mismatch-decomposition})
gives the ratio at one sampled token:
\begin{equation}
 \hat w
 =\underbrace{\frac{\pi_\theta(o_t^i\mid h_t^i)}
                   {\pi^{\mathrm{dense}}_{\theta_{\mathrm{old}}}(o_t^i\mid h_t^i)}}_{w(\theta)}
  \underbrace{\frac{\pi^{\mathrm{dense}}_{\theta_{\mathrm{old}}}(o_t^i\mid h_t^i)}
                   {\pi^{\mathrm{sparse}}_{\theta_{\mathrm{old}},s}(o_t^i\mid h_t^i)}}_{w_{s,t}}.
\end{equation}
The first factor accounts for the learner's change from the dense rollout
snapshot; the second corrects the current token conditional at $h_t^i$.
A single $w_{s,t}$ leaves the sparse distribution of the prefix and
continuation uncorrected.
CTPO addresses prefix state--action mismatch through cumulative token
importance ratios~\citep{DBLP:journals/corr/abs-2605-07331}.
Our analysis additionally examines the dependence of a response's
GRPO advantage on other sampled group members.

Sparse-RL combines response filtering with tokenwise
reweighting~\citep{DBLP:conf/acl/LuoZHZWSSLZ26}, while
Sparrow pairs mismatch-guided sparsity scheduling with truncated importance
sampling~\citep{DBLP:journals/corr/abs-2606-08446}.
For other rollout mismatches, IcePop masks tokens with large
inference--training probability discrepancies~\citep{DBLP:journals/corr/abs-2510-18855},
and AIS adapts importance correction for quantized
rollouts~\citep{DBLP:journals/corr/abs-2605-13907}.
These interventions mitigate mismatch without generally restoring the dense
response-group law. To isolate the role of group-relative advantages in
direct-sparse rollout, we grant each complete response its exact importance
weight.

Let $P$ and $Q_s$ be the dense and direct-sparse
laws defined in Section~\ref{sec:preliminary-learning-implications}.
If $Q_s(o\mid x)>0$ whenever $P(o\mid x)>0$, the chain rule gives the exact
response weight
\begin{equation}
 W_i=\frac{P(o^i\mid x)}{Q_s(o^i\mid x)}
    =\prod_{t=1}^{|o^i|}w_{s,t}.
 \label{eq:ssrl-full-response-weight}
\end{equation}
The standard importance-sampling identity~\citep{agapiou2017importance}
gives $\mathbb E_{Q_s}[W_i f(o^i)]=\mathbb E_P[f(o^i)]$ for any integrable
function $f$ of one complete response. For groups of independently sampled responses,
the joint likelihood ratio is $\prod_{j=1}^{G}W_j$.
Weighting only response $i$'s update by $W_i$ leaves its group members
under $Q_s$, which can bias the GRPO update even with exact response weights
(Section~\ref{app:ssrl-group-correction}). \modelname instead preserves the
dense response-group law before computing rewards and advantages
so that its policy importance ratio remains $w(\theta)$
(Section~\ref{app:ssrl-sampling-assumptions}).

\subsection{Limits of response-wise correction in GRPO}
\label{app:ssrl-group-correction}

\begin{proof}[Proof of Proposition~\ref{prop:ssrl-response-correction}]
Fix the prompt $x$ and learner parameters. Suppressing the conditioning
on $x$, let $P(u)$ be the dense rollout probability of a complete
response $u$, and $Q_s(u)$ the probability obtained by sampling the
sparse drafter directly. The corresponding exact sparse-to-dense weight is
$W(u)=P(u)/Q_s(u)$, with $P(u)>0$ implying $Q_s(u)>0$ and all
expectations finite. Here an update is the
group-averaged surrogate gradient. Both groups use the same GRPO
surrogate referenced to the dense rollout policy; $W$ multiplies each
response gradient as a fixed outer factor.

For two independent responses $U,V$, let $g(u,v)$ be the gradient
contribution of $u$ after its advantage is computed from the pair
$(u,v)$. The group gradient averages $g(u,v)$ and $g(v,u)$.
By exchangeability, its expected value equals the expected contribution
of the first response, also when each contribution is weighted by its
own $W$. Let $H(v)=\sum_u P(u)g(u,v)$ be the expected gradient of a focal
response drawn from the dense law $P$ when its comparison is fixed at
$v$. Write $G_{\mathrm{corr}}$ for the expected gradient after weighting
direct-sparse groups response by response, and $G_{\mathrm{dense}}$ for
the dense-rollout baseline. Then
\begin{equation}
 \begin{aligned}
 G_{\mathrm{corr}}
 &=\sum_{u,v}Q_s(u)Q_s(v)W(u)g(u,v)\\
 &=\sum_{u,v}P(u)Q_s(v)g(u,v)
  =\sum_v Q_s(v)H(v),\\
 G_{\mathrm{dense}}
 &=\sum_{u,v}P(u)P(v)g(u,v)
  =\sum_v P(v)H(v).
 \end{aligned}
 \label{eq:ssrl-app-response-change-measure}
\end{equation}
In the weighted expectation, $W(u)$ replaces the focal factor $Q_s(u)$
with $P(u)$. The focal response is thus accounted for under the dense
rollout law $P$, while its comparison response still follows the
direct-sparse law $Q_s$. For larger groups, the other
$G-1$ responses likewise retain this sparse law. A difference in expected
gradients occurs when $H(v)$ varies with the comparison response and the
shift from the dense law $P$ to the direct-sparse law $Q_s$ does not cancel.

\paragraph{A GRPO counterexample.}
Consider three one-token responses $a,b,c$ with rewards $(0,1,2)$.
In $(a,b,c)$ order, take the dense law $P=(1/3,1/3,1/3)$ and the
direct-sparse law $Q_s=(1/2,1/3,1/6)$, giving exact response weights
$(W(a),W(b),W(c))=(2/3,1,2)$.
For two-response groups, the sample standard deviation of the rewards is
$\sigma(u,v)=|r(u)-r(v)|/\sqrt{2}$. With stabilizer
$\varepsilon>0$, the group-normalized advantage of $u$ is
\begin{equation}
 A(u,v)
 =\frac{r(u)-[r(u)+r(v)]/2}{\sigma(u,v)+\varepsilon}
 =\frac{r(u)-r(v)}{\sqrt{2}|r(u)-r(v)|+2\varepsilon}.
\end{equation}
Take $(\pi_\theta(a),\pi_\theta(b),\pi_\theta(c))
=(1,e^\theta,1)/(2+e^\theta)$ at $\theta=\theta_{\mathrm{old}}=0$.
Then $\pi_0=P$, and the log-probability derivatives
$\left.\partial_\theta\log\pi_\theta(u)\right|_{\theta=0}$ are
$(-1/3,2/3,-1/3)$. The clipped policy ratio is one at this point,
so under the no-KL, no-entropy surrogate used in our experiments the
one-token gradient contribution is
$g(u,v)=A(u,v)\left.\partial_\theta\log\pi_\theta(u)\right|_{\theta=0}$,
holding the sampled advantage fixed.

Let $c_d=d/(\sqrt{2}d+2\varepsilon)$ be the magnitude of the advantage
for a reward gap $d\in\{1,2\}$. Against $a$, responses $b$ and $c$
contribute $2c_1/9$ and $-c_2/9$ to $H(a)$; symmetry then gives
\begin{equation}
 \bigl(H(a),H(b),H(c)\bigr)
 =\frac{2c_1-c_2}{9}(1,0,-1).
\end{equation}
The dense comparison law $P$ is uniform, so its positive and negative
terms cancel. Under the direct-sparse law $Q_s$, the low-reward
comparison $a$ occurs more often than the high-reward comparison $c$,
yielding
\begin{equation}
 \begin{aligned}
 G_{\mathrm{dense}}&=0,\\
 G_{\mathrm{corr}}
 &=(Q_s(a)-Q_s(c))H(a)
  =\frac{2c_1-c_2}{27}
  =\frac{\sqrt{2}}{27(\sqrt{2}+2\varepsilon)
                        (\sqrt{2}+\varepsilon)}>0.
 \end{aligned}
 \label{eq:ssrl-app-counterexample-update}
\end{equation}
Even though $W(b)=1$, its expected contribution changes when the
comparison response is drawn from the direct-sparse law $Q_s$ rather
than the dense law $P$. Thus exact response-wise weights do not
generally recover the dense expected gradient.
\end{proof}

Here $Q_s$ is the complete-response law of direct sparse rollout. In
\modelname, $q_s$ supplies token proposals, while the completed responses
follow $P$, as proved next.

\subsection{Dense response-group preservation under adaptive sparse proposals}
\label{app:ssrl-sampling-assumptions}

\begin{proof}[Proof of Proposition~\ref{prop:ssrl-exactness}]
Fix a committed prefix $(x,o_{<t}^{i})$ and any drafter state determined by
earlier sampling and verification. Let $p(v)$ be the dense verifier's
probability of token $v$ at this prefix, computed with dense historical KV,
and let $q_s(v)$ be the proposal probability recorded when the candidate is
sampled. Earlier acceptance decisions do not use later candidates, so
conditioning on their acceptance leaves the next candidate distributed
according to its recorded $q_s$ at the committed prefix.

The candidate is accepted with probability $\min\{1,p(v)/q_s(v)\}$;
on rejection, the discarded suffix is replaced by a draw from the
residual distribution in Eq.~(\ref{eq:ssrl-residual-sampling}). The total
rejection probability is $Z=\sum_u[p(u)-q_s(u)]_+$, so for $Z>0$ the next
committed token has probability
\begin{equation}
 \min\{p(v),q_s(v)\}+Zr_s(v)
 =\min\{p(v),q_s(v)\}+[p(v)-q_s(v)]_+
 =p(v).
 \label{eq:ssrl-app-one-token-exactness}
\end{equation}
If $Z=0$, then $p=q_s$ and every candidate is accepted. The residual
also covers tokens with $q_s(v)=0$ when $p(v)>0$.

The identity holds conditional on every preceding drafter state, including
states selected using completed verification. Dense prefill and the bonus
token after a fully accepted block also follow $p$; rejected suffixes are
rolled back before subsequent verification. Thus every committed token has
the dense conditional distribution, even as sparse proposals adapt.
With the same termination rules, the chain rule gives
\begin{equation}
 \Pr_{\text{\modelname},s}(o^{i}\mid x)
 =\prod_{t=1}^{|o^{i}|}
   \pi^{\mathrm{dense}}_{\theta_{\mathrm{old}}}(o_t^{i}\mid x,o_{<t}^{i})
 =P(o^{i}\mid x).
 \label{eq:ssrl-app-trajectory-exactness}
\end{equation}
All requested responses retain their group identities and use independent
randomness. Their joint law is therefore $\prod_{i=1}^{G}P(o^{i}\mid x)$,
as in Eq.~(\ref{eq:ssrl-group-distribution}). Applying the same reward,
advantage, and training operations to these groups produces gradient
estimates with the same distribution as dense AR rollouts.
\end{proof}

For direct sparse rollout, Eq.~(\ref{eq:preliminary-mismatch-decomposition})
gives the training ratio $\hat w=w(\theta)w_s$. For \modelname,
Eq.~(\ref{eq:ssrl-app-one-token-exactness}) instead makes the committed
token's behavior probability $p(v)$, so at $v=o_t^{i}$ the ratio is
\begin{equation}
 \begin{aligned}
 \frac{\pi_\theta(v\mid x,o_{<t}^{i})}
      {\Pr_{\text{\modelname},s}(v\mid x,o_{<t}^{i})}
 &=\frac{\pi_\theta(v\mid x,o_{<t}^{i})}
         {\pi^{\mathrm{dense}}_{\theta_{\mathrm{old}}}(v\mid x,o_{<t}^{i})}\\
 &=w(\theta).
 \end{aligned}
\end{equation}
Thus the extra $w_s$ factor in Eq.~(\ref{eq:preliminary-mismatch-decomposition})
does not enter the training update: $q_s$ supplies proposals, while the
committed tokens follow $p$.

\section{Implementation Details of \modelname}
\label{app:ssrl-implementation}

This appendix details the sparse drafting, full-attention verification,
and supporting RL execution design in Section~\ref{sec:methodology}.
Our rollout backend inherits the open-source implementation of
\citet{DBLP:journals/corr/abs-2512-01278} under 
veRL~\citep{DBLP:conf/eurosys/ShengZYWZZPL025} framework.
We first specify sparse-context updates and verification state, then
describe request scheduling, and finally connect their execution to
the RL training loop in Algorithm~\ref{alg:ssrl-rollout}.

\subsection{Sparse-context construction and verification state}
\label{app:ssrl-sparse-context}

\paragraph{Budget and recent positions.}
For each layer $\ell$, let $\mathcal C_\ell$ denote the KV indices
retained after exact sampling: the pre-draft context and the accepted
draft prefix; KV-head indices remain implicit. The implementation applies
a minimum budget $K_{\min}$ and caps the ratio-based budget by the
available context length. The number of positions retained at context
refresh is
\begin{equation}
 K_\ell=\min\bigl\{|\mathcal C_\ell|,
 \max\{K_{\min},\lceil s|\mathcal C_\ell|\rceil\}\bigr\}.
 \label{eq:ssrl-sparse-budget}
\end{equation}
If the preceding verification commits $n$ output tokens, including its
correction or bonus token, $\mathcal R_\ell$ contains the last
$\min\{K_\ell,n+1\}$ valid cached positions around the verification boundary.
The newly sampled correction or bonus token has no KV entry at this
boundary; its KV is computed by subsequent execution.
Each layer and KV head independently selects the remaining
$K_\ell-|\mathcal R_\ell|$ positions from
$\mathcal C_\ell\setminus\mathcal R_\ell$ using the verifier scores
$\mathbf a_\ell$. This is Eq.~(\ref{eq:ssrl-draft-selection}) with
$K_\ell$ replacing $\lceil s|\mathcal C_\ell|\rceil$.
The selected historical positions remain fixed within a draft block,
while $\mathcal I_\ell$ is extended with
newly computed KV entries.
Thus $K_\ell$ specifies the budget at refresh; the attended context grows
with subsequent draft steps.

\paragraph{Attention operators.}
Each draft step applies attention to the selected KV entries:
\begin{equation}
 \operatorname{SparseAttn}(\mathbf q;\mathcal I)
 =\operatorname{softmax}\!\left(
 \frac{\mathbf q\mathbf K_{\mathcal I}^{\top}}{\sqrt d}
 \right)\mathbf V_{\mathcal I},
 \label{eq:ssrl-sparse-attention}
\end{equation}
where $\mathbf q\in\mathbb R^{1\times d}$ is the current query and $d$
the head dimension. For the verifier's query matrix $\mathbf Q$,
full attention is
\begin{equation}
 \operatorname{FullAttn}(\mathbf Q)
 =\operatorname{softmax}_{\mathrm{causal}}\!\left(
 \frac{\mathbf Q\mathbf K^{\top}}{\sqrt d}
 \right)\mathbf V,
 \label{eq:ssrl-full-attention}
\end{equation}
where $\mathbf K$ and $\mathbf V$ cover the complete context and
the proposed block.

\paragraph{Attention-score reduction.}
For each layer and KV head, $\mathbf a_\ell$ assigns each historical
position the mean of its normalized attention weights over all query
positions in the verified block and the query heads sharing that KV head.
Queries for candidate tokens later rejected by exact sampling still
contribute to this mean. A separate GPU kernel recomputes the necessary
query--key scores using the verifier's saved softmax normalizers and stores
only the reduced per-position scores. GPU Top-K consumes these scores;
neither step materializes the full attention matrix or requires another
Transformer forward.

\paragraph{Shared KV and proposal state.}
Sparse and full attention address the same KV storage; the sparse
selection is an index view into the complete cache. Verification
overwrites provisional draft KV entries with dense states. Once exact
sampling determines the accepted draft prefix, KV entries for the
unaccepted candidate suffix are discarded. For the current unverified
block, each request also retains the full-vocabulary proposal distributions
on the GPU. These distributions and the verifier's target distributions supply
acceptance and residual sampling
(Appendix~\ref{app:ssrl-sampling-assumptions}). The proposal buffer is
reused after verification and released when the request completes.

\subsection{KV-aware AIMD scheduler and mixed-stage execution}
\label{app:ssrl-scheduling}

\paragraph{Concurrency update rule.}
At scheduler iteration $j$, the request-concurrency limit $C_j$ determines
whether additional requests can join the active batch of size $B$. Let $u_j$ denote GPU
KV-cache utilization and $O_j$ the number of paused requests awaiting
KV restoration from CPU.
The backlog captures outstanding demand that GPU utilization alone
misses after caches move to CPU. The KV-aware AIMD scheduler updates $C_j$ using
additive increase and multiplicative decrease:
% Preserve the OR condition and the low-utilization branch priority.
\begin{equation}
C_{j+1}=\operatorname{clip}_{[C_{\min},C_{\max}]}
\left(
\begin{cases}
C_j+\delta, & u_j<u_{\mathrm{low}},\\
\lfloor\beta C_j\rfloor,
  & u_j>u_{\mathrm{high}}\ \vee \ O_j>O_{\mathrm{th}},\\
C_j, & \text{otherwise}.
\end{cases}
\right)
\label{eq:ssrl-aimd}
\end{equation}
Here $\delta$ is a positive integer increment, $\beta\in(0,1)$ the
reduction factor, and $C_{\min},C_{\max}$ the concurrency bounds, with
utilization thresholds $u_{\mathrm{low}}<u_{\mathrm{high}}$.
The branches are evaluated in order: low utilization permits additive
growth; otherwise, high utilization or an excessive backlog triggers
multiplicative backoff. If $B$ exceeds a reduced $C_j$, additional requests wait
for $B$ to fall below the new limit, while already active requests
continue subject to KV capacity.

\paragraph{Request scheduling and KV capacity.}
ShadowKV combines low-rank key storage, value offloading, and sparse KV
reconstruction to increase long-context inference
throughput~\citep{DBLP:conf/icml/SunCB0Z0DCC25}.
Here, offloading preserves complete dense KV entries for subsequent
full-attention verification.
The KV-aware AIMD scheduler operates over a KV offload/restore mechanism.
Each rollout replica maintains pending and active requests, plus a backlog
of paused requests awaiting KV restoration from CPU.
The concurrency limit $C_j$ determines when additional requests can join the active batch, while a separate capacity
check accounts for the KV storage required by the next execution batch.
When resident KV state leaves insufficient space for execution, the
scheduler prioritizes eligible requests with the smallest current KV-cache
length, pauses them, and transfers KV pages to CPU memory. 
Token history, proposal state, and execution stage remain
available for resumption. A paused request rejoins the active set after its KV
state has been restored, subject to capacity and the concurrency limit.
Restoring paused requests takes priority over starting new requests. Completed requests release their
resources. The scheduler observes GPU KV utilization and the KV-restoration
backlog and applies Eq.~(\ref{eq:ssrl-aimd}) to set the next iteration's
concurrency limit. A reduction of this limit delays request activation while
already active requests continue, subject to KV-capacity checks.

\paragraph{Mixed execution.}
We use mixed execution of prefill, draft, and verification requests.
The active requests can occupy different stages. A prefill request
contributes a prompt chunk, a drafting request contributes one query,
and a verifying request contributes its candidate block and preceding
committed token. These queries are packed into one model forward, with
sparse attention for drafting and full attention for prefill and
verification. Consequently, the active request count differs from the
query count processed by the model. This execution implements the phase
mixing analyzed in Appendix~\ref{app:why-ssrl-cost}.
CPU scheduling overlaps GPU execution; each request starts its next draft
block once verification and sparse-index updates are complete.

\subsection{Execution in the RL training loop}
\label{app:ssrl-rl-integration}

At each RL step, the actor's parameters are synchronized to
the rollout replicas, and this snapshot supplies both drafting and
verification throughout generation. After complete response groups return
to the trainer with their prompt and group identities, rollout resources
are released for the actor update.

\paragraph{Resource lifetimes across RL steps.}
\label{app:ssrl-state-lifetimes}
Table~\ref{tab:ssrl-state-lifetimes} tracks the transition from rollout
generation to actor optimization and back. CUDA graphs use predefined
batch-size buckets. After KV resources are released, captures tied to
their storage are renewed for these buckets before the next rollout,
allowing graph replay for the covered batch sizes.

\begin{table}[htbp]
\centering
\small
\caption{Resource lifetimes across the rollout and actor-update phases.}
\label{tab:ssrl-state-lifetimes}
\begin{tabular}{@{}p{0.20\linewidth}p{0.35\linewidth}p{0.39\linewidth}@{}}
\toprule
State & Before actor update & Before the next rollout \\
\midrule
Requests and KV & Release request state and KV resources.
& Start new requests with KV capacity fixed for the rollout. \\
Rollout model & Retain parameter buffers on GPU.
& Synchronize the updated actor snapshot into these buffers. \\
CUDA graphs & Discard captures tied to released KV storage; retain the planned batch sizes.
& Recapture the planned sizes for graph replay during generation. \\
\bottomrule
\end{tabular}
\end{table}

Algorithm~\ref{alg:ssrl-rollout} combines this resource lifecycle with the
per-request drafting and verification steps. Here $\mathcal B_j$ is the
scheduled batch, $m$ counts proposals in the current draft block, and
$o^{i}$ is the committed response prefix for request $(x,i)$.
The symbols $\mathcal I$, $\mathcal R$, and $\mathbf a$ denote the
layer- and KV-head-specific state defined above.
\textsc{KVSchedule} and \textsc{MixedForward} implement
Appendix~\ref{app:ssrl-scheduling}; \textsc{ExactSample} returns the
committed tokens under Appendix~\ref{app:ssrl-sampling-assumptions}.

\begin{algorithm}[h]
\caption{\modelname Pipeline}
\label{alg:ssrl-rollout}
\small
\algrenewcommand\algorithmicrequire{\textbf{Input:}}
\algtext*{EndFor}
\algtext*{EndIf}
\begin{algorithmic}[1]
\Require Policy $\theta$, training prompts
\Statex \hspace*{-\algorithmicindent}\textbf{Parameters:} group size $G$, draft length $k$, sparse budget $(s,K_{\min})$;
  KV-aware AIMD scheduler parameters in Eq.~(\ref{eq:ssrl-aimd})
\For{each RL step}
  \State $\theta_{\mathrm{old}}\gets\theta$; synchronize rollout replicas
  \State Prepare fixed-capacity KV storage and CUDA graphs for planned batch sizes
  \State Initialize $G$ requests per prompt; distribute across replicas
  \ForAll{rollout replicas in parallel}
    \For{scheduler iterations $j=0,1,\ldots$ until rollout completes}
      \State $\mathcal B_j\gets\textsc{KVSchedule}(C_j)$ \Comment{Capacity, offload, restore, start requests}
      \State $(q,p,\mathbf a)\gets\textsc{MixedForward}_{\theta_{\mathrm{old}}}(\mathcal B_j;\mathcal I)$
        \Comment{Eqs.~(\ref{eq:ssrl-sparse-attention})--(\ref{eq:ssrl-full-attention})}
      \ForAll{requests $(x,i)$ in $\mathcal B_j$}
        \If{prefill completes}
          \State $o_1^{i}\sim p;\quad o^{i}\gets[o_1^{i}]$
          \State $\mathcal I\gets\{1,\ldots,S\};\quad m\gets0$
        \ElsIf{a draft step completes}
          \State $m\gets m+1;\quad \tilde o_m^{i}\sim q_m$ \Comment{Verify when $m=k$}
        \ElsIf{verification completes}
          \State $\Delta o^{i}\gets\textsc{ExactSample}(\tilde o_{1:k}^{i},q_{1:k},p_{1:k+1})$
          \State $o^{i}\gets[o^{i};\Delta o^{i}]$; discard unaccepted draft KV
          \State $S\gets\text{valid KV length};\quad K\gets\min\{S,\max\{K_{\min},\lceil sS\rceil\}\}$
          \State $\mathcal R\gets\{S-\min(K,|\Delta o^{i}|+1)+1,\ldots,S\}$
          \State $\mathcal I\gets\mathcal R\cup\operatorname{TopK}_{K-|\mathcal R|}(\mathbf a);\quad m\gets0$
            \Comment{Eq.~(\ref{eq:ssrl-draft-selection})}
        \EndIf
        \State Collect completed responses; release their state
      \EndFor
      \State Observe $(u_j,O_j)$; update $C_{j+1}$ by Eq.~(\ref{eq:ssrl-aimd})
    \EndFor
  \EndFor
  \State Release request and KV resources for actor optimization
  \State Compute rewards and group-relative advantages from complete response groups
  \State Update $\theta$ using the policy objective and these advantages
\EndFor
\end{algorithmic}
\end{algorithm}

\section{Detailed Experimental Setup}
\label{app:experimental-setup}
% Experimental protocol confirmed by the authors on 2026-09-14.
% Historical launch scripts and figure-preparation notes are not configuration sources.

This appendix specifies the training, decoding, and measurement settings
for Section~\ref{sec:experiments}, together with the baseline implementations
and ablation configurations.

\subsection{Infrastructure, Models, and Data}

\paragraph{Infrastructure.}
We run the main training experiments on a single node with eight NVIDIA
H100 GPUs, 64 CPU cores, and 1\,TB of host memory. We use
veRL~\citep{DBLP:conf/eurosys/ShengZYWZZPL025} v0.7.0 for training and
vLLM~\citep{DBLP:conf/sosp/KwonLZ0ZY0ZS23} v0.12.0 for the baseline
rollout engines. Actor training uses
Fully Sharded Data Parallel (FSDP). Rollout generation uses four replicas,
each with tensor parallelism of two (TP2/DP4). All methods share the GPU
allocation, KV-memory budget, and a maximum decoding batch size of 256
requests per replica. CUDA graphs, decoding warm-up, prefix caching, and
chunked prefill are enabled for every method.

\paragraph{Models and training data.}
We initialize the target policies from Qwen3-1.7B and
Qwen3-8B~\citep{DBLP:journals/corr/abs-2505-09388}, with thinking enabled
and the default Qwen3 chat template. We use the original
DAPO-MATH-17k dataset~\citep{DBLP:conf/nips/YuZZYZYDFLLLLLL25} without
additional filtering and shuffle it with a fixed random seed. Answer
extraction follows DeepScaleR~\citep{deepscaler2025}. Responses that
exceed the length limit, are truncated, or contain no valid extracted
answer are scored as incorrect.

\subsection{Training and Method Configuration}

\paragraph{RL training.}
We use the DAPO variant of GRPO with 256 prompts per training batch and
eight responses per prompt, yielding 2,048 responses per step. Sampling
and group filtering are disabled, so the generated batch is also the
batch used for the actor update. The mini-batch size is 32, with dynamic
batching enabled. We use AdamW with a learning rate of $10^{-5}$ and
DAPO's asymmetric clipping parameters
$(\epsilon_{\mathrm{low}},\epsilon_{\mathrm{high}})=(0.2,0.28)$.
The overlong penalty uses a 4K-token window. We use no KL loss, entropy
bonus, or additional importance correction in the main comparison.

\paragraph{Sampling and evaluation.}
Training rollouts use temperature 1.0 without top-$p$ or top-$k$ filtering.
The maximum response lengths are 16K tokens for Qwen3-1.7B and 24K tokens
for Qwen3-8B. For evaluation, we use temperature 1.0 and top-$p=0.7$.
We report AIME26 and ArXiv-Math pass@1 scores from the best checkpoint,
estimated using 16 sampled responses per problem as detailed in
Appendix~\ref{app:experimental-metrics}.

\paragraph{\modelname.}
We run \modelname as the rollout backend in veRL
(Appendix~\ref{app:ssrl-rl-integration}).
Both target models use a draft length of $k=6$ and a sparse budget ratio
of $s=0.05$. The draft length counts only proposed tokens; the correction
or bonus token is additional. The sparse KV budget is computed from the
current context length, with a minimum budget of 192 tokens, following
Appendix~\ref{app:ssrl-sparse-context}. The KV-aware AIMD scheduler in
Appendix~\ref{app:ssrl-scheduling} uses
$(C_{\min},C_{\max})=(160,256)$, $\delta=5$, $\beta=0.7$,
$(u_{\mathrm{low}},u_{\mathrm{high}})=(0.8,0.95)$, and
$O_{\mathrm{th}}=3$. Concurrency limits apply to each rollout replica.

\subsection{Measurement and Reporting}
\label{app:experimental-metrics}

\paragraph{Evaluation datasets.}
We use the MathArena releases~\citep{dekoninck2026matharena} of both
evaluation benchmarks. AIME26 contains all 30 problems from the 2026
American Invitational Mathematics Examination, distributed as the
\texttt{train} split of \texttt{MathArena/aime\_2026}.\footnote{AIME26:
\url{https://huggingface.co/datasets/MathArena/aime_2026}.}
ArXiv-Math denotes the 103-problem ArXivMath collection covering
December 2025 through March 2026, distributed as the \texttt{train} split
of \texttt{MathArena/arxivmath}.\footnote{ArXivMath:
\url{https://huggingface.co/datasets/MathArena/arxivmath}.}
This collection contains 17, 23, 32, and 31 problems from December,
January, February, and March, respectively. We evaluate on all problems
in each benchmark.

\paragraph{Evaluation sampling and aggregation.}
For each evaluated checkpoint, we run one evaluation per benchmark and
sample 16 responses per problem at temperature 1.0 and top-$p=0.7$.
This yields 480 responses on AIME26 and 1,648 on ArXiv-Math per checkpoint.
Let $N$ be the number of problems in a benchmark and $c_i$ the number of
correct responses among the 16 samples for problem $i$. The reported
pass@1 score is
\begin{equation}
  \widehat{\mathrm{pass@1}}
  = \frac{100}{N}\sum_{i=1}^{N}\frac{c_i}{16}\quad (\%).
  \label{eq:eval-pass-at-one}
\end{equation}
Thus, we average the sampled correctness rate for each problem across
the benchmark, giving every problem equal weight. Each benchmark is
aggregated separately.

\paragraph{Reporting window.}
The rollout time, training-step time, throughput, and block efficiency in
Table~\ref{tab:main_results} cover the full training run, excluding the
first step for warm-up. All methods use the same reporting window within
each model size.

\paragraph{Timing and throughput.}
Rollout time measures response generation. Training-step time measures
the complete step, including reward evaluation, log-probability and
advantage computation, actor updates, weight synchronization, and
method-specific work such as historical-response cache updates.
For each step, training throughput is
$N_{\mathrm{tok}}/(8T_{\mathrm{step}})$ in tokens/s/GPU, where
$N_{\mathrm{tok}}$ counts non-padding prompt and response tokens in the
batch. We report the arithmetic mean of these per-step throughputs.
Generation throughput in the KV-management ablation uses rollout time
in the denominator with the same token and GPU accounting. Throughput
speedups are ratios of the corresponding means relative to AR.

\paragraph{Scope of performance attribution.}
Table~\ref{tab:main_results} evaluates the complete RL system against
the specified baseline execution stacks. The lifecycle mechanisms in
Appendix~\ref{app:ssrl-state-lifetimes} support repeated colocated
execution, but their individual speed contributions are not isolated.
Table~\ref{tab:kv_aware_aimd_scheduler_1_7b} compares KV management and request scheduling
within the sparse-speculative stack; it is not an ablation of the
cross-step release and restoration paths.

\paragraph{Speculative decoding and KV diagnostics.}
Block efficiency $\tau$ is measured by directly counting the output
tokens emitted per speculative round, including the correction or bonus
token, and averaging over rounds in the reporting window. We also report
draft-token acceptance rates. The KV-management ablation records peak
CPU KV-cache occupancy, offload episodes, reoffloading, and transferred
pages per 1,000 generated response tokens. Transfers include both
device-to-host offloading and host-to-device reloading.

\subsection{Baseline Implementations}
\label{app:experimental-baselines}

\paragraph{AR.}
The AR baseline uses full-attention autoregressive decoding in vLLM with
speculative decoding disabled. It is the timing and throughput reference
for each target model.

\paragraph{N-gram Lookup.}
We use vLLM v0.12.0's prompt-lookup speculative decoding, following
\citet{saxena2023prompt}. It matches the current token suffix against
earlier occurrences in the prompt or generated response and proposes
the following tokens. We retain the default lookup configuration,
with both the minimum and maximum matching lengths set to five
tokens.\footnote{vLLM v0.12.0 speculative configuration:
\url{https://docs.vllm.ai/en/v0.12.0/api/vllm/config/speculative/}.}
Each speculative round proposes at most five draft tokens ($k=5$).
With this fixed proposal-length limit, positions after the first rejected
token still incur target-model computation during parallel verification,
although their outputs are discarded~\citep{DBLP:journals/corr/abs-2601-11580}.
Low acceptance therefore limits cost amortization: the observed block
efficiencies of $\tau=1.31$ and $1.29$ in Table~\ref{tab:main_results}
leave verification cost spread over only about 1.3 emitted tokens per
round, limiting rollout gains despite inexpensive lookup-based proposals.

\paragraph{History Retrieval.}
We implement RhymeRL~\citep{DBLP:journals/corr/abs-2508-18588} using
veRL's SpecRL recipe. A distributed suffix-tree cache stores historical
responses by prompt and is updated asynchronously after rollout
generation. Each request queries the cache with up to seven trailing
tokens of its current sequence, using a minimum retrieval token
probability of 0.1. The proposal length adapts to the number of output
tokens emitted in the preceding verification round. Retrieved tokens
are verified as deterministic proposals using rejection sampling with
residual correction.

\paragraph{Auxiliary Proposer.}
This baseline represents the learned auxiliary drafting approach used in
FastRL~\citep{DBLP:conf/asplos/0004YGY0GC0K026} and
NeMo RL~\citep{DBLP:journals/corr/abs-2604-26779}.
We use the pretrained EAGLE3~\citep{DBLP:conf/nips/LiWZZ25} checkpoints
\texttt{Qwen3-1.7B\_eagle3} and \texttt{Qwen3-8B\_eagle3} released by
AngelSlim on Hugging Face.\footnote{\url{https://huggingface.co/AngelSlim/Qwen3-1.7B_eagle3}
and \url{https://huggingface.co/AngelSlim/Qwen3-8B_eagle3}.}
Both use chain drafting with a fixed draft length of $k=5$ in the
veRL--vLLM stack. The auxiliary weights remain
frozen throughout the main experiments as the target policy is updated.
This comparison focuses on rollout acceleration without additional
drafter training. Checkpoint portability and online co-training are
examined separately in Appendix~\ref{app:checkpoint-portability}.

\subsection{Ablation Configurations}

\paragraph{Direct sparse + mismatch correction.}
This baseline in
Figure~\ref{fig:training_stability_ablation}(a) removes speculative decoding
from our sparse rollout system and applies Sparse-RL's correction methods.
It uses the main experiment's model, data, optimizer, batching, and
sparse-budget settings.
It generates responses with the sparse policy and applies the corrections
in Sections~4.2--4.3 of \citet{DBLP:conf/acl/LuoZHZWSSLZ26}:
sequence-level masking for anomalously small dense-to-sparse probability
ratios and tokenwise importance reweighting outside the clipped policy
objective. This is the correction-based sparse-rollout variant discussed
in Section~\ref{sec:preliminary-policy-mismatch} and
Appendix~\ref{app:ssrl-correction-limits}.

\paragraph{KV management and request scheduling.}
Table~\ref{tab:kv_aware_aimd_scheduler_1_7b} uses ten RL steps after warm-up
under the same training and decoding configuration as the main
Qwen3-1.7B experiment. Recompute-256 restores evicted KV states by
recomputation, whereas Offload-128, Offload-160, and Offload-256 use
host-memory offloading with fixed per-replica concurrency limits of 128,
160, and 256. The KV-aware AIMD scheduler uses the main experiment's
offloading mechanism and concurrency update rule with the parameters above. All four variants prioritize
requests with the smallest current KV-cache length for transfer to CPU,
following Appendix~\ref{app:ssrl-scheduling}. 

% Manuscript fragment: input after \appendix, with graphicx and natbib loaded.
% The standalone wrapper overrides the figure path before reading this file.
\providecommand{\modelname}{S3R}
\providecommand{\cotrainfigpath}{figs/co_training/}

\section{Checkpoint Portability and Online Drafter Adaptation}
\label{app:checkpoint-portability}

An RL workflow may start from an independently post-trained checkpoint for
which a matched auxiliary proposer is unavailable. We examine two practical
aspects of this setting: applying speculative decoding to different target
checkpoints without additional drafter training, and maintaining an auxiliary
proposer as the target policy evolves. These experiments complement the main
training results by evaluating the flexibility of sparse self-drafting.
Figure~\ref{fig:checkpoint-portability} compares checkpoint portability
and periodic online drafter adaptation.

% Three panels with one shared caption; requires graphicx and subcaption.
% Vector panels are prepared by figs/prepare_caption_panels.py.
\begin{figure}[htbp]
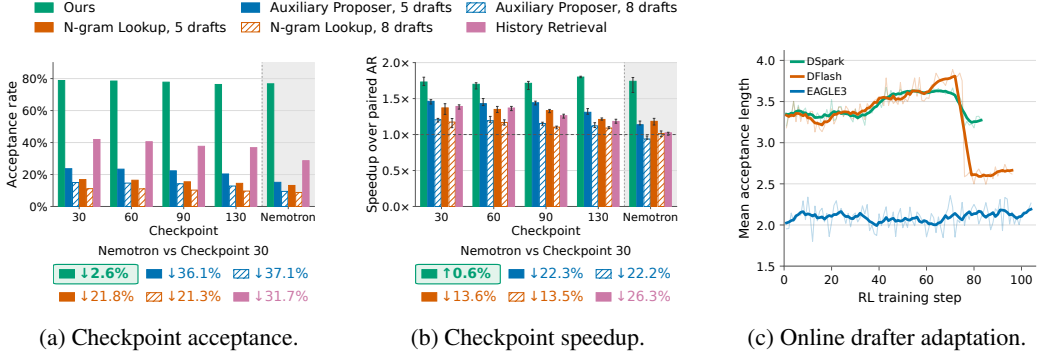

  \centering
  \includegraphics[width=0.65\linewidth]{\cotrainfigpath ../caption_panels/checkpoint_legend.pdf}\hfill\mbox{}
  \par\smallskip
  \newlength{\cotrainPanelHeight}
  \setlength{\cotrainPanelHeight}{0.247\linewidth}
  \begin{subfigure}[t]{0.32\linewidth}
    \centering
    \includegraphics[height=\cotrainPanelHeight]{\cotrainfigpath ../caption_panels/checkpoint_acceptance.pdf}
    \caption{Checkpoint acceptance.}\label{fig:checkpoint-acceptance}
  \end{subfigure}\hfill
  \begin{subfigure}[t]{0.32\linewidth}
    \centering
    \includegraphics[height=\cotrainPanelHeight]{\cotrainfigpath ../caption_panels/checkpoint_speedup.pdf}
    \caption{Checkpoint speedup.}\label{fig:checkpoint-speedup}
  \end{subfigure}\hfill
  \begin{subfigure}[t]{0.32\linewidth}
    \centering
    \includegraphics[height=\cotrainPanelHeight]{\cotrainfigpath ../caption_panels/online_adaptation.pdf}
    \caption{Online drafter adaptation.}\label{fig:online-drafter-adaptation}
  \end{subfigure}
  \caption{Checkpoint portability and periodic online drafter adaptation.
  (\subref{fig:checkpoint-acceptance}) Draft-token acceptance and
  (\subref{fig:checkpoint-speedup}) generation speedup relative to paired AR across target checkpoints;
  error bars summarize placement variability checked against additional runs.
  (\subref{fig:online-drafter-adaptation}) Mean acceptance length during periodic co-training,
  with faint observations and bold seven-step rolling means.}
  \label{fig:checkpoint-portability}
\end{figure}

\subsection{Checkpoint Portability without Auxiliary Proposer Training}
\label{app:checkpoint-evaluation}

\paragraph{Targets and setup.}
We evaluate five Qwen3-8B-family targets: steps 30, 60, 90 and 130 from a
DAPO--History Retrieval run; and the
released Nemotron-Cascade-8B model~\citep{nemotronCascade2025}. The latter has
undergone both supervised fine-tuning and reinforcement learning. These
targets are treated as separate categories: Nemotron do not
extend the training trajectory of steps 30--130.

All targets use the same 30 AIME24 problems, with 16 responses per problem,
temperature 0.6. This checkpoint study uses default $S^3R$ as the setup in our main experiment.
Auxiliary Proposer uses
the fixed \texttt{AngelSlim/Qwen3-8B\_eagle3} drafter. Auxiliary Proposer and N-gram Lookup
propose five draft tokens, while History Retrieval uses adaptive suffix-cache
drafting with a maximum of 16 tokens. All methods are evaluated without
additional drafter fitting or online drafter updates.
For this checkpoint study, History Retrieval's prebuilt checkpoint-0 trajectory cache
is held fixed during target evaluation.

We exclude RL optimization time and are distinct from the main paper's
end-to-end training speedups.
Draft-token acceptance measures the fraction of proposed tokens accepted.
The direction glyphs below each figure summarize each method's relative
change from DAPO step 30 to Nemotron, computed as
$(v_{\mathrm{Nemotron}}/v_{\mathrm{DAPO30}}-1)\times100\%$ for the
corresponding metric $v$. These annotations compare separate target
categories rather than a continuous training trajectory.

\paragraph{Results.}
Across the five targets, \modelname{} retains 76.32--78.78\% draft-token
acceptance and 1.698--1.799$\times$ generation speedup
(Figure~\ref{fig:checkpoint-portability}). Within the shared steps-30--90
trajectory, acceptance decreases modestly from 78.78\% to 77.76\%, while
speedup remains at 1.698--1.731$\times$. On Nemotron, \modelname{} achieves
76.74\% acceptance and 1.740$\times$ speedup without fitting a new drafter.
Fixed Auxiliary Proposer remains usable on this target: five drafts yield
1.136$\times$ speedup, whereas eight drafts yield 0.938$\times$.
Thus, direct reuse can provide reduced and configuration-sensitive
benefits. The result supports the practical value of obtaining proposals
directly from a selected supported checkpoint; it does not establish that
retraining an auxiliary proposer is necessary for every target.

\subsection{Periodic Online Drafter Adaptation}
\label{app:online-drafter-adaptation}

\paragraph{Online drafter setting and training.}
We analyze three Qwen3-8B GRPO baselines on DAPO Math in our
VERL--SpeCo implementation, using EAGLE3, DFlash~\citep{dflash2026},
and the DSpark adaptation recipe~\citep{dspark2026}. Each run uses one
node with 16 NPU workers, five rollouts per prompt, and prompt/response
limits of 512/8,192 tokens. Target policies are optimized with AdamW at
a learning rate of $10^{-6}$. Target hidden states are collected through
the old-logprob forward pass, and the auxiliary drafter is trained every
five RL steps per update. EAGLE3 uses a prompt batch size of 64, TP$=2$, and
three draft tokens; DFlash uses a batch size of 64, TP$=2$, and 15 draft
tokens; DSpark uses a batch size of 64, TP$2$, and seven draft tokens.
Thus, drafter optimization is periodic rather than accompanying every
target-policy update.
Increasing drafter optimization times will increase system overhead.

\paragraph{Observations and implications.}
Figure~\ref{fig:online-drafter-adaptation} shows the mean acceptance
length $\tau$ for each run. We use this metric
to describe acceptance changes within each run.
In these configurations, periodic adaptation does not produce a monotonic
increase in proposal acceptance. Comparing the same windows, steps 61--75
and 76--82, mean acceptance length decreases from 3.603 to 3.245 for
DSpark and from 3.761 to 2.607 for DFlash, while EAGLE3 remains
approximately stable, changing from 2.075 to 2.124. We assume that
even co-training the drafter and the target policy requires specific
adaptation to avoid drafter staleness problems.

Successful online co-training has sbeen reported in
NeMo-RL~\citep{onlineDraftCoTraining2026}, which supports EAGLE3, DFlash,
and DSpark through branch-aware context parallelism and cross-stage
target-feature transport, accelerating RL while closely tracking the baseline
learning trajectory. The present case concerns the periodic adaptation
configuration used here. Maintaining acceptance can
require coordinating auxiliary proposer adaptation with policy evolution.
\modelname{} removes the need to select and initialize a separate drafter
and schedule its optimization: its sparse proposals and dense verification
share the synchronized rollout-policy snapshot. Acceptance can still vary
with the sparse approximation and the target distribution, but this design
avoids a separate source of parameter lag and auxiliary training effort.

\section{Additional Rollout Acceleration Diagnostics}
\label{app:rollout-acceleration-diagnostics}

We supplement the main results with reward and response-length trajectories
for both model sizes, acceptance measurements across draft positions, and
the Qwen3-1.7B throughput and timing results.

\subsection{Training Dynamics across Model Sizes}
\label{app:training-dynamics}

\begin{figure}[htbp]
  \centering
  \begin{subfigure}[t]{0.49\linewidth}
    \centering
    \includegraphics[width=\linewidth,trim=12bp 12bp 44bp 10bp,clip]{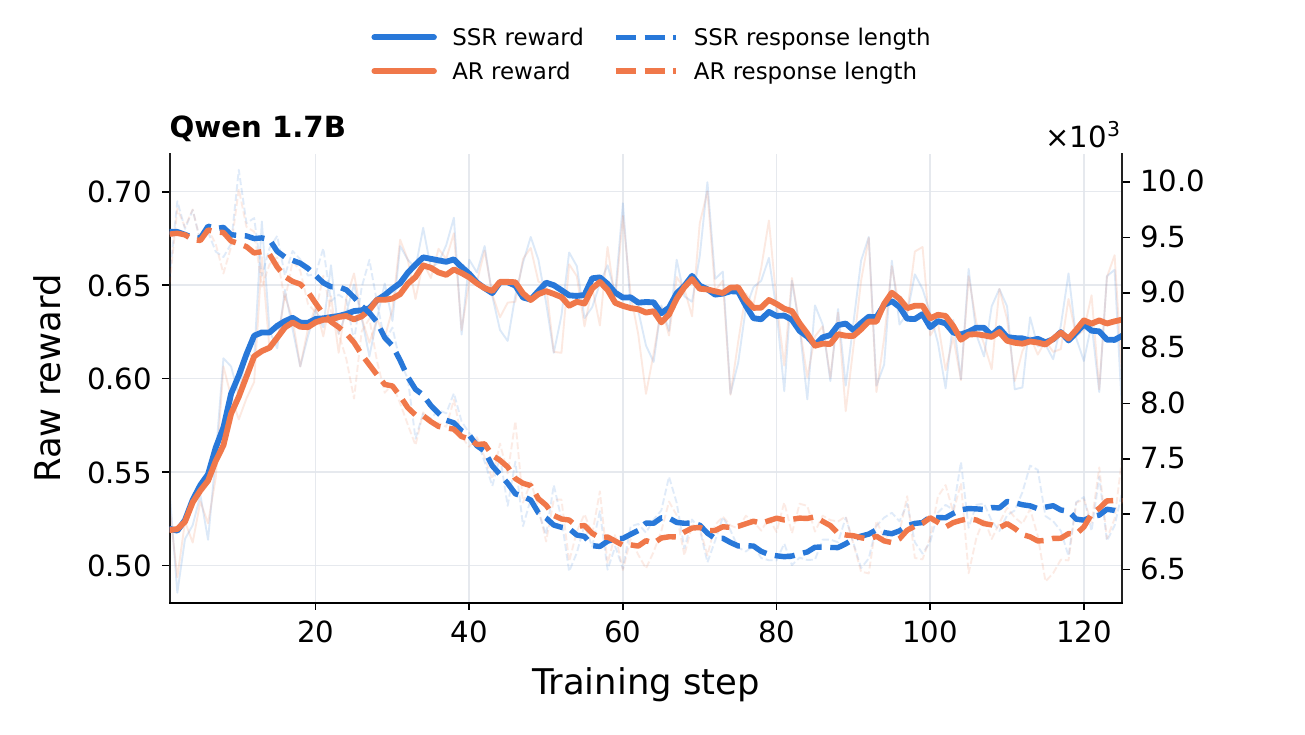}
    \caption{Qwen3-1.7B.}\label{fig:reward-response-1-7b}
  \end{subfigure}\hfill
  \begin{subfigure}[t]{0.49\linewidth}
    \centering
    \includegraphics[width=\linewidth,trim=38bp 12bp 30bp 10bp,clip]{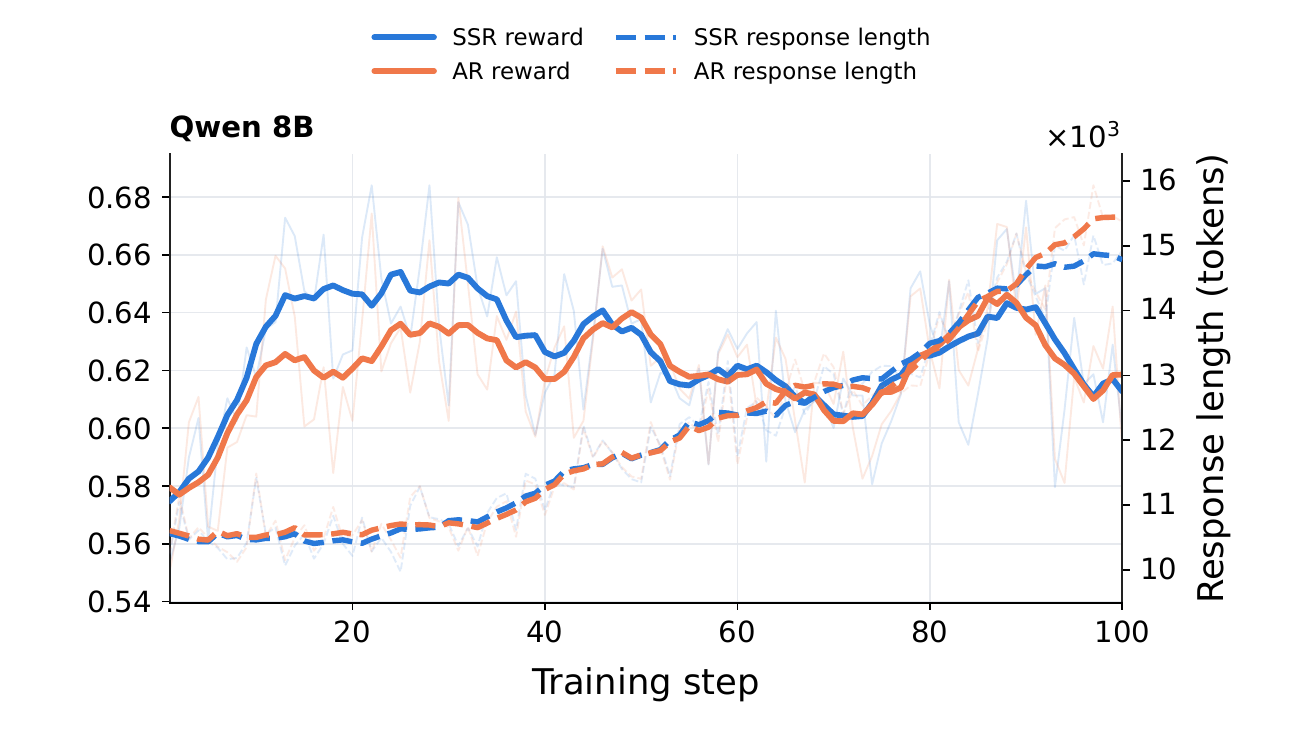}
    \caption{Qwen3-8B.}\label{fig:reward-response-8b}
  \end{subfigure}
  \caption{Reward and response-length trajectories of AR and \modelname across training steps on Qwen3-1.7B and Qwen3-8B.}
  \label{fig:reward_response_length}
\end{figure}

Figure~\ref{fig:reward_response_length} shows comparable reward and
response-length trajectories for AR and \modelname across both model sizes.
On Qwen3-1.7B, reward rises rapidly before leveling off, while responses
become shorter. Qwen3-8B exhibits the opposite length trend, with responses
growing throughout training; despite some early reward differences, the
later reward trajectories remain close.
\modelname follows AR in both cases, showing that its acceleration is
compatible with similar learning behavior under both decreasing and
increasing response lengths.

\subsection{Acceptance across Draft Positions}
\label{app:rollout-position-acceptance}

% Keep the diagnostics on text pages instead of a separate float-only page.
\begin{figure}[!ht]
  \centering
  \begin{subfigure}[t]{0.328\linewidth}
    \vspace{0pt}\centering
    \includegraphics[width=\linewidth]{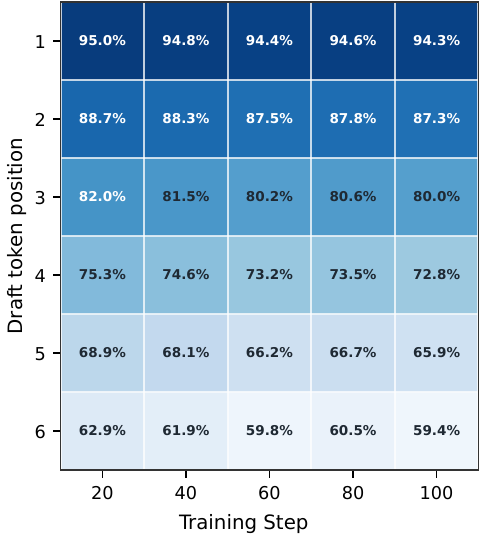}
    \caption{Qwen3-1.7B.}\label{fig:position-acceptance-1-7b}
  \end{subfigure}\hspace{0.006\linewidth}%
  \begin{subfigure}[t]{0.294\linewidth}
    \vspace{0pt}\centering
    \includegraphics[width=\linewidth]{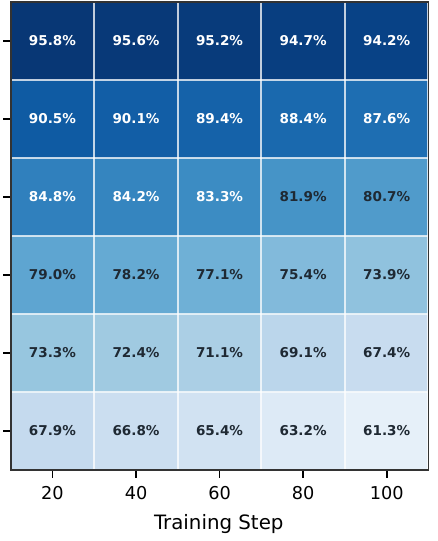}
    \caption{Qwen3-8B.}\label{fig:position-acceptance-8b}
  \end{subfigure}\hspace{0.005\linewidth}%
  \begin{minipage}[t]{0.08\linewidth}
    \vspace{0pt}\centering
    \includegraphics[width=\linewidth]{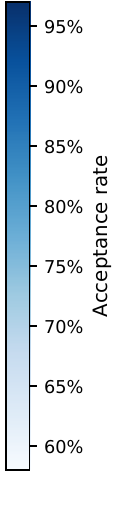}
  \end{minipage}
  \caption{\modelname draft-token acceptance across training checkpoints for
  (\subref{fig:position-acceptance-1-7b}) Qwen3-1.7B and
  (\subref{fig:position-acceptance-8b}) Qwen3-8B.
  Cells report acceptance at each draft position.}
  \label{fig:rollout_position_acceptance}
\end{figure}

Draft-token acceptance is strongly position dependent
(Figure~\ref{fig:rollout_position_acceptance}). For Qwen3-8B at step 20,
acceptance falls from $95.8\%$ for the first proposed token to $67.9\%$ for
the sixth. From step 20 to step 100, first-position acceptance decreases by
$1.6$ percentage points, whereas sixth-position acceptance decreases by
$6.6$ points, and the checkpoint-level rate falls from $81.9\%$ to $77.5\%$.
Qwen3-1.7B exhibits the same ordering: acceptance at the first and sixth
positions changes from $95.0\%$ and $62.9\%$ to $94.3\%$ and $59.4\%$,
respectively, corresponding to declines of $0.7$ and $3.5$ percentage points.
Thus, later speculative positions have both lower acceptance and larger
declines during training.
It is highly related to the entropy dynamics, showing that the RL is encouraging exploration.
This pattern motivates exploring an adaptive
draft horizon that shortens when late-position acceptance deteriorates,
which we left for future work.

\subsection{Qwen3-1.7B Throughput and Timing--Acceptance Results}
\label{app:rollout-acceleration-1-7b}

% Keep the figure below its subsection heading and beside its discussion.
\begin{figure}[H]
  \centering
  \begin{subfigure}[t]{0.4\linewidth}
    \centering
    \includegraphics[width=\linewidth]{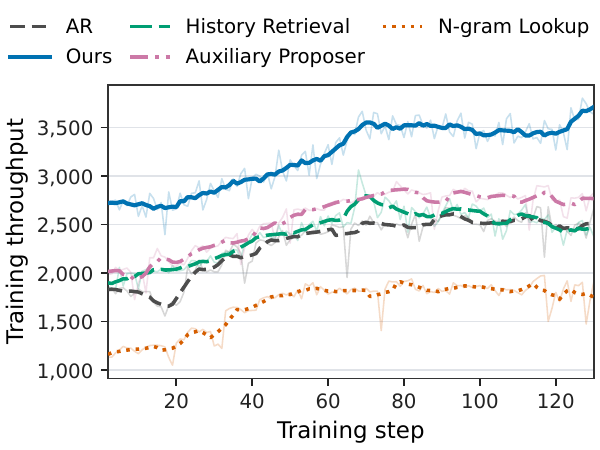}
    \caption{Training throughput.}\label{fig:rollout_throughput_1_7b}
  \end{subfigure}\hfill
  \begin{subfigure}[t]{0.4\linewidth}
    \centering
    \includegraphics[width=\linewidth]{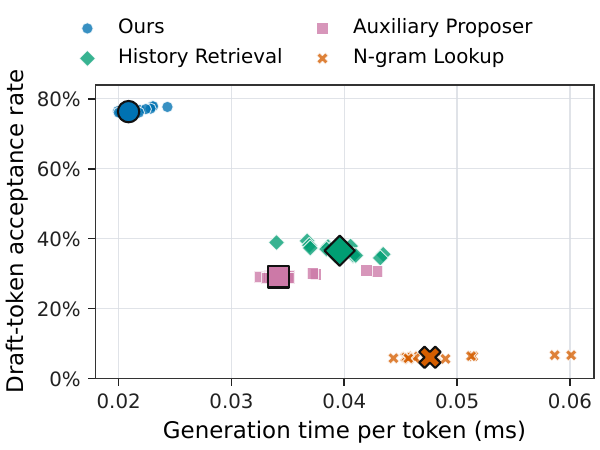}
    \caption{Timing--acceptance trade-off.}\label{fig:rollout_timing_acceptance_1_7b}
  \end{subfigure}
  \caption{Qwen3-1.7B rollout diagnostics.
  (\subref{fig:rollout_throughput_1_7b}) End-to-end training throughput;
  (\subref{fig:rollout_timing_acceptance_1_7b}) generation time per processed token
  versus draft-token acceptance; outlined markers denote method medians.}
  \label{fig:rollout_acceleration_1_7b}
\end{figure}

\paragraph{Throughput.}
End-to-end throughput counts processed prompt and response tokens over the
complete training-step time and GPU count; speedups are ratios of mean
per-step throughput. Figure~\ref{fig:rollout_throughput_1_7b}
shows the throughput trajectory during training. The mean throughput in
Table~\ref{tab:main_results} is $3{,}213.8$ tokens/s/GPU for \modelname and
$2{,}308.6$ for AR, a $1.392\times$ ratio. Auxiliary Proposer, History
Retrieval, and N-gram Lookup reach $1.110\times$, $1.045\times$, and
$0.724\times$ the AR throughput, respectively. 

\paragraph{Acceptance and generation cost.}
In Figure~\ref{fig:rollout_timing_acceptance_1_7b}, \modelname has median draft-token
acceptance of $76.3\%$ at $0.0209$ ms per processed token, compared with
$29.1\%$ at $0.0341$ ms for Auxiliary Proposer and $36.6\%$ at $0.0396$ ms
for History Retrieval. Its median generation-time metric is $38.9\%$ lower
than Auxiliary Proposer's and $47.3\%$ lower than History Retrieval's.
Although History Retrieval has higher acceptance than Auxiliary Proposer,
it also incurs greater time per processed token. This agrees with the 8B
comparison: acceptance must be considered together with proposal,
verification, and runtime costs.

% Systems rationale and measurements for sparse self-speculative RL rollouts.
% This file is self-contained; the previous measurement narrative is retained
% in sparse_self_speculation_measurements.tex for the .bak.tex reference version.

% Keep prose and equations compact when the next figure starts on a new page.
\raggedbottom

\Needspace{8\baselineskip}
\section{Why Self-Sparse Speculative Decoding for RL Rollouts?}
\label{app:why-ssrl}

Sparse self-drafting is useful when the attention work it removes outweighs
its retained projection work and dense verification cost. This appendix
supports Section~\ref{sec:efficiency-runtime} by connecting that cost balance
to measured rollout workloads. We first derive the latency condition, then
trace the evidence from attention savings to realized throughput, and use a
simple speedup surface to explain the budget--acceptance trade-off.

\subsection{Latency decomposition for self-sparse speculative decoding}
\label{app:why-ssrl-cost}

We extend the serialized latency accounting in
Eq.~\ref{eq:efficiency-sd-accounting} by separating projection work from
historical KV reads at locally fixed active batch size $B$ and mean cached
context length $S$. For complete rounds, let $\alpha$ be the mean fraction
of the $k$ draft tokens accepted; then $\tau=1+k\alpha\in[1,k+1]$.

Let $T_{\mathrm{GEMM}}(B)$ cover QKVO, FFN, and LM-head work, and let
$T_{\mathrm{Attn}}(M)$ cover attention at aggregate KV volume
$M\propto B\cdot S$. The shared-weight drafter retains projection work but
reads a fraction $s\in(0,1]$ of the historical KV.
For bandwidth-bound attention at fixed effective bandwidth, the Roofline
model~\citep{williams2009roofline} motivates approximating latency as linear
in $M$. Verification processes
the $k$ drafts plus the preceding round's final token, whose dense KV is
not yet cached. Assuming $k+1\ll S$ and shared historical reads across
these $(k+1)B$ queries gives
\begin{equation}
\begin{aligned}
 T_p&=T_{\mathrm{GEMM}}(B)+T_{\mathrm{Attn}}(M),\\
 T_q&=T_{\mathrm{GEMM}}(B)+sT_{\mathrm{Attn}}(M),\\
 T_V&=T_{\mathrm{GEMM}}((k+1)B)+T_{\mathrm{Attn}}(M).
\end{aligned}
\label{eq:why-ssrl-sparse-stages}
\end{equation}
Projection latency need not scale linearly with $B$ because small batches
reuse weights. Larger batches can increase $T_V/T_p$, whereas longer
contexts at fixed $B$ push it toward one
\citep{sadhukhan2025magicdec,DBLP:journals/corr/abs-2606-18967}.

Shared weights also permit draft and verification queries in one forward.
With requests evenly staggered across $k$ draft stages and one verification
stage~\citep[Sections~3.2--3.3]{DBLP:journals/corr/abs-2512-01278},
each iteration has $kB/(k+1)$ single-query drafts and $B/(k+1)$
verifications of $k+1$ queries each. This yields $(2k+1)B/(k+1)$ queries
and $(ks+1)M/(k+1)$ effective KV reads per iteration. Each request
completes a round in $k+1$ iterations, so the mixed-stage cost per output
token is
\begin{equation}
 \frac{T_{\mathrm{mix}}}{\tau}
 =\frac{k+1}{\tau}\left[
 T_{\mathrm{GEMM}}\!\left(\frac{2k+1}{k+1}B\right)
 +T_{\mathrm{Attn}}\!\left(\frac{ks+1}{k+1}M\right)\right].
 \label{eq:why-ssrl-mixed-latency}
\end{equation}
By linearity, the attention term becomes $(ks+1)T_{\mathrm{Attn}}(M)/\tau$.

At full acceptance, $\tau=k+1$ maximizes output per round for fixed
$(B,S,s,k)$. Comparing Eq.~\ref{eq:why-ssrl-mixed-latency} with the AR
latency $T_p$ gives the ideal mixed-stage feasibility condition
\begin{equation}
 k(1-s)T_{\mathrm{Attn}}(M)
 >(k+1)\left[
 T_{\mathrm{GEMM}}\!\left(\frac{2k+1}{k+1}B\right)
 -T_{\mathrm{GEMM}}(B)\right].
 \label{eq:why-ssrl-full-acceptance}
\end{equation}
The left-hand side is the attention saving; the right-hand side is the
projection overhead from mixing draft and verification queries. If the
condition fails, no acceptance rate yields a speedup in this component
model. Observed acceleration also depends on actual acceptance and costs
beyond this approximation, examined in
Appendix~\ref{app:why-ssrl-calibration}.

\subsection{From attention savings to rollout speedup}
\label{app:why-ssrl-profiles}
\label{app:why-ssrl-setup}

The measurements below follow the cost balance in three steps: identify
where attention dominates, account for the drafting and verification costs,
and test whether the resulting savings improve output throughput.

\paragraph{Experimental setting.}
Table~\ref{tab:why-ssrl-setup} summarizes the component-profiling setup.
Here $B$ is the active batch on one TP replica, $S$ is context length per
request, and $N_{\mathrm{KV}}=B S$ is the aggregate live KV volume, with
$1\mathrm{K}=1024$ tokens. The subsequent throughput sweep uses real
mathematical continuations and stochastic sampling, as specified below.

\begin{table}[h]
\centering
\small
\caption{Settings for the component measurements in Tables~\ref{tab:why-ssrl-e1-breakdown}--\ref{tab:why-ssrl-e2-stages} and Figure~\ref{fig:why-ssrl-e1-ratio}.}
\label{tab:why-ssrl-setup}
\begin{tabular}{ll}
\hline
Item & Setting \\
\hline
Model & Qwen3-8B, DAPO step-60 checkpoint \\
Hardware & $2\times$NVIDIA H100 80GB HBM3 \\
Parallelism & one rollout-engine replica, TP${=}2$ \\
Precision & BF16 \\
Workload & equal-length synthetic prompts, fixed concurrency, ignore EOS \\
Sampling & temperature 0 (deterministic systems workload) \\
AR baseline & ordinary AR with full attention \\
Sparse draft & Pillar attention, budget ratio 0.05, minimum 192 tokens \\
Speculation & draft length $k=6$, exact full-model verification \\
\hline
\end{tabular}
\end{table}

\subsubsection{Larger live KV volumes make attention the bottleneck}
\label{app:why-ssrl-e1}

At small live KV volumes, sparse attention targets only a modest share of
decoding cost. The first four rows of Table~\ref{tab:why-ssrl-e1-breakdown}
use the rollout-tail coordinates studied by
EfficientRollout~\citep{DBLP:journals/corr/abs-2606-18967}.
Attention accounts for 18.8--22.4\% of the measured kernel time, while
FFN, QKVO, and LM-head projections together account for 62.1--65.0\%.
As context grows at fixed $B=16$, attention rises from 2.34 ms at $S=4$K
to 12.93 ms at $S=32$K, while FFN stays near 2.24 ms.
This transition explains why \modelname targets historical KV reads.
EfficientRollout's quantized self-drafter reduces weight traffic in
projection-dominated tails; sparse self-drafting addresses the growing
attention cost at larger live KV volumes. The two motivations are consistent
with different workload regimes, so the relevant drafting strategy depends
on which cost dominates.

\begin{table}[h]
\centering
\small
\setlength{\tabcolsep}{5.5pt}
\caption{Full-attention AR decoding costs: ms per iteration (share of GPU kernel time),
excluding TP collectives. The upper rows cover small-KV rollout-tail workloads;
the lower rows show the transition to attention-dominated decoding.
QKVO denotes query, key, value, and output projections.}
\label{tab:why-ssrl-e1-breakdown}
\begin{tabular}{@{}rrrrrrrr@{}}
\hline
$N_{\mathrm{KV}}$ & $S$ & $B$ & Attention & FFN & QKVO & LM head & Other \\
\hline
8K   & 2K  & 4   & 0.96 (19.4\%)  & 2.18 (43.8\%) & 0.83 (16.6\%) & 0.21 (4.1\%) & 0.80 (16.1\%) \\
8K   & 4K  & 2   & 0.93 (18.8\%)  & 2.18 (44.2\%) & 0.82 (16.6\%) & 0.20 (4.2\%) & 0.80 (16.2\%) \\
8K   & 8K  & 1   & 0.98 (19.8\%)  & 2.17 (43.7\%) & 0.82 (16.5\%) & 0.20 (4.1\%) & 0.79 (16.0\%) \\
16K  & 2K  & 8   & 1.16 (22.4\%)  & 2.19 (42.1\%) & 0.83 (16.0\%) & 0.21 (4.0\%) & 0.81 (15.6\%) \\
\hline
64K  & 4K  & 16  & 2.34 (36.2\%)  & 2.24 (34.6\%) & 0.85 (13.2\%) & 0.21 (3.3\%) & 0.82 (12.7\%) \\
128K & 8K  & 16  & 3.85 (48.4\%)  & 2.23 (28.1\%) & 0.85 (10.7\%) & 0.21 (2.6\%) & 0.82 (10.3\%) \\
256K & 16K & 16  & 6.85 (62.4\%)  & 2.24 (20.4\%) & 0.85 (7.8\%)  & 0.21 (1.9\%) & 0.82 (7.5\%) \\
512K & 32K & 16  & 12.93 (75.8\%) & 2.24 (13.2\%) & 0.85 (5.0\%)  & 0.21 (1.2\%) & 0.82 (4.8\%) \\
512K & 16K & 32  & 12.98 (75.1\%) & 2.37 (13.7\%) & 0.87 (5.1\%)  & 0.21 (1.2\%) & 0.84 (4.9\%) \\
512K & 8K  & 64  & 13.11 (76.2\%) & 2.23 (13.0\%) & 0.76 (4.4\%)  & 0.21 (1.2\%) & 0.89 (5.2\%) \\
512K & 4K  & 128 & 13.07 (75.2\%) & 2.32 (13.4\%) & 0.80 (4.6\%)  & 0.22 (1.3\%) & 0.97 (5.6\%) \\
512K & 2K  & 256 & 13.44 (72.3\%) & 2.73 (14.7\%) & 0.99 (5.4\%)  & 0.24 (1.3\%) & 1.18 (6.3\%) \\
\hline
\end{tabular}
\end{table}

Figure~\ref{fig:why-ssrl-e1-ratio} extends this comparison across feasible
batch sizes $B\in\{1,4,16,32,64,128,256\}$ and context lengths
$S\in\{1,2,4,8,16,24,32\}\mathrm{K}$.
Attention-to-FFN latency ratios are similar across different factorizations
of the same $N_{\mathrm{KV}}$: about 0.43--0.45 at 8K, 0.97--1.12 at 64K,
and 4.92--5.88 at 512K. Thus, in this profiled range, aggregate live KV
volume is a useful first-order indicator of the attention bottleneck.
With eight responses per prompt in our RLVR setup
(Section~\ref{sec:setup}), concurrency can make attention expensive even at
moderate context lengths. Conversely, a
shrinking active batch can reduce this opportunity despite long remaining
responses. These profiles motivate tracking both batch and context when
deciding where sparse speculation is useful within an RL rollout, 
which is left for future work.

\begin{figure}[h]
\centering
\includegraphics[width=0.68\linewidth]{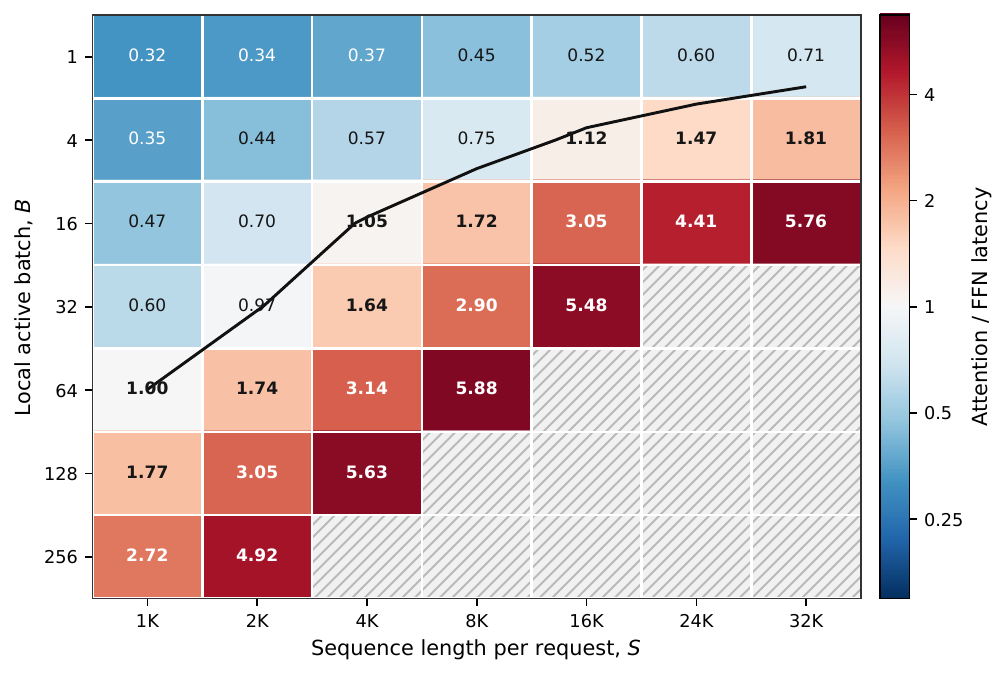}
\caption{Attention-to-FFN kernel latency ratio for full-attention AR decoding
across batch sizes and context lengths.
The black contour marks equal latency; hatched cells exceed the measured KV-capacity limit.}
\label{fig:why-ssrl-e1-ratio}
\end{figure}

\subsubsection{Attention savings must cover drafting and verification}
\label{app:why-ssrl-e2}

Table~\ref{tab:why-ssrl-e2-attention} measures attention time for sparse
drafting plus dense verification at eight feasible $(B,S)$ settings.
With a draft budget of $s=0.05$, the reduction grows from 38.3\% at
$N_{\mathrm{KV}}=64$K to 63.2--63.7\% at 512K.
The savings persist with dense verification included, supporting
\modelname's decision to confine sparsity to proposals while paying for
faithful target sampling. They reduce attention work without removing the
complete KV state required for verification, which explains the need for
the KV-aware runtime in Section~\ref{sec:runtime-support}.
Whether these component savings translate into throughput depends on
the retained projection work and accepted output.

\begin{table}[h]
\centering
\small
\setlength{\tabcolsep}{6pt}
\caption{Attention time in AR and \modelname, in ms per scheduler iteration per TP rank
(median over three runs). \modelname includes sparse drafting and dense verification,
averaged over the full window, including idle iterations. These component savings
must still be weighed against accepted output and other execution costs.}
\label{tab:why-ssrl-e2-attention}
\begin{tabular}{@{}rrrrrr@{}}
\hline
$B$ & $S$ & $N_{\mathrm{KV}}$ & Full attention & \modelname attention & Reduction \\
\hline
16 & 4K  & 64K  & 2.31 & 1.42 & 38.3\% \\
16 & 8K  & 128K & 3.83 & 1.89 & 50.6\% \\
16 & 16K & 256K & 6.82 & 2.79 & 59.2\% \\
32 & 4K  & 128K & 3.80 & 1.96 & 48.5\% \\
32 & 8K  & 256K & 6.83 & 2.85 & 58.3\% \\
32 & 16K & 512K & 12.91 & 4.68 & 63.7\% \\
64 & 4K  & 256K & 6.91 & 2.90 & 58.1\% \\
64 & 8K  & 512K & 13.02 & 4.79 & 63.2\% \\
\hline
\end{tabular}
\end{table}

\Needspace{8\baselineskip}
Table~\ref{tab:why-ssrl-e2-stages} adds these stage costs.
Under the sequential accounting in Eq.~\ref{eq:efficiency-sd-accounting},
a round costs $6T_q+T_V$ and produces $1+6\alpha$ output tokens per request.
Equating this cost to $(1+6\alpha)T_p$ gives the break-even acceptance
\begin{equation}
 \alpha_{\mathrm{BE}}=\frac{1}{6}\left(\frac{6T_q+T_V}{T_p}-1\right).
 \label{eq:why-ssrl-stage-threshold}
\end{equation}
At 64K--128K live KV tokens, the measured stage costs exceed the AR
reference even at full acceptance. At 256K, the threshold falls to
75.8--77.0\%, and at 512K to 45.5--46.7\%.
The low-KV cases expose a limit of sharing policy weights: it avoids
separate drafter updates, but retains projection work at every draft step.
Where full acceptance is insufficient under this accounting, better
proposal agreement alone cannot deliver acceleration. A future runtime
could reduce stage costs or fall back to AR in such states, reserving
sparse speculation for favorable workloads. These thresholds use
separately timed sequential stages; complete-window timing below measures
the realized benefit.

\begin{table}[h]
\centering
\small
\setlength{\tabcolsep}{6pt}
\caption{Drafting and verification costs for $k=6$ (ms, medians over three runs).
The AR reference $7T_p$ generates as many tokens as a fully accepted round.
$\alpha_{\mathrm{BE}}$ is the acceptance needed to break even under sequential
stage accounting. AR and drafting use CUDA Graphs; verification runs eagerly.
The summed costs exclude idle iterations and cross-request phase mixing.}
\label{tab:why-ssrl-e2-stages}
\begin{tabular}{@{}rrrrrrrr@{}}
\hline
$B$ & $S$ & $N_{\mathrm{KV}}$ & AR ($7T_p$) & Draft ($6T_q$) & Verify ($T_V$) & Total & $\alpha_{\mathrm{BE}}$ \\
\hline
16 & 4K  & 64K  & 54.42 & 40.06 & 28.06 & 68.12 & $>100\%$ \\
16 & 8K  & 128K & 65.55 & 39.15 & 27.97 & 67.12 & $>100\%$ \\
16 & 16K & 256K & 88.51 & 42.38 & 28.10 & 70.48 & 76.2\% \\
32 & 4K  & 128K & 66.70 & 40.89 & 29.03 & 69.92 & $>100\%$ \\
32 & 8K  & 256K & 88.12 & 43.01 & 26.86 & 69.87 & 75.8\% \\
32 & 16K & 512K & 133.18 & 43.73 & 28.56 & 72.29 & 46.7\% \\
64 & 4K  & 256K & 87.22 & 42.05 & 27.96 & 70.01 & 77.0\% \\
64 & 8K  & 512K & 129.85 & 42.51 & 26.68 & 69.19 & 45.5\% \\
\hline
\end{tabular}
\end{table}

\subsubsection{The sparse budget trades accepted progress against execution cost}
\label{app:why-ssrl-budget}

To test this trade-off directly, we sweep the configured draft budget over
3\%, 5\%, 10\%, 20\%, 35\%, and 50\% for both Qwen3-1.7B and Qwen3-8B.
We use frozen DAPO step-60 checkpoints, two H100 GPUs with TP$=2$, BF16,
$k=6$, and a minimum draft budget of 192 tokens. Each run starts with
$B=64$ requests at $S=8$K, using three fixed bundles of mathematical
continuations sampled at temperature one. AR and \modelname use the same
engine and inputs. Timing covers complete draft--verify rounds at fixed
concurrency, with EOS disabled and no KV offloading. Speedup is the ratio
of committed output tokens per second to the matched AR reference.
This isolates the budget comparison from policy updates and request admission.

Figure~\ref{fig:why-ssrl-budget-overlay} shows why acceptance alone cannot
select the budget. Increasing the budget from 5\% to 50\% raises
representative acceptance from 79.5\% to 95.8\% for 1.7B and from 80.4\%
to 96.7\% for 8B, yet median speedup falls from $1.533\times$ to
$1.094\times$ and from $1.370\times$ to $1.058\times$, respectively.
Reducing the budget to 3\% also lowers speedup, to $1.469\times$ and
$1.307\times$. The 5--10\% region provides a favorable balance in these
workloads: further savings at smaller budgets lose useful progress, while
larger budgets add more cost than their improved acceptance recovers.
The sweep supports our default of 5\%, without identifying a unique optimum.

Under Proposition~\ref{prop:ssrl-exactness}, this throughput-driven budget
choice preserves dense response groups. This is the practical value of
\modelname's separation of proposal and behavior policies, given the
instability of direct sparse rollout at the same 5\% budget
(Section~\ref{sec:stability}). It motivates adapting the budget to accepted
output per unit time as the policy and workload evolve.

\begin{figure}[h]
 \centering
 \begin{subfigure}[t]{0.48\linewidth}
  \centering
  \includegraphics[width=\linewidth]{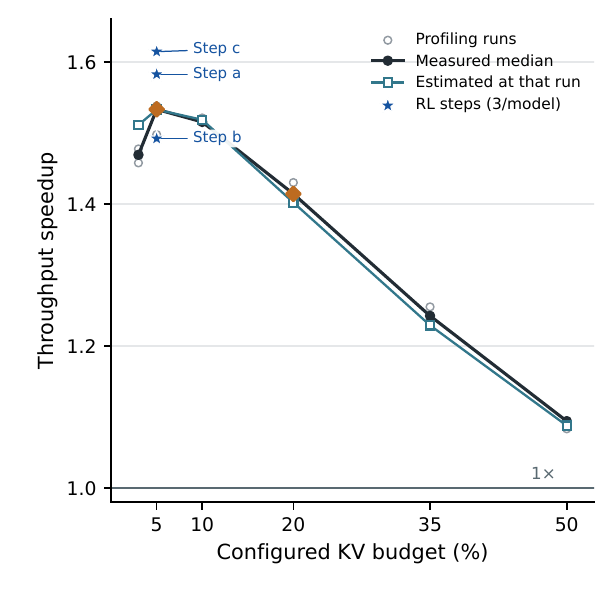}
  \caption{Qwen3-1.7B.}\label{fig:why-ssrl-budget-overlay-1-7b}
 \end{subfigure}\hfill
 \begin{subfigure}[t]{0.48\linewidth}
  \centering
  \includegraphics[width=\linewidth]{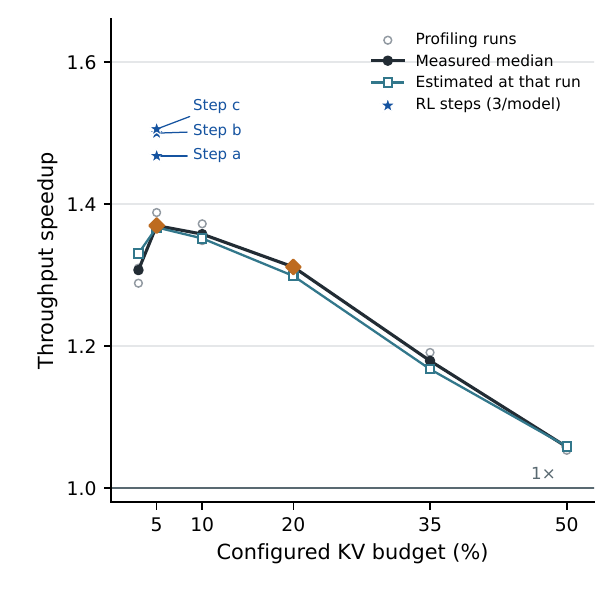}
  \caption{Qwen3-8B.}\label{fig:why-ssrl-budget-overlay-8b}
 \end{subfigure}
 \caption{Sparse-budget sweeps for
 (\subref{fig:why-ssrl-budget-overlay-1-7b}) Qwen3-1.7B and
 (\subref{fig:why-ssrl-budget-overlay-8b}) Qwen3-8B at $B=64$, initial $S=8$K, and $k=6$.
 Black circles show median measured speedup; squares show estimates from
 Eq.~\ref{eq:why-ssrl-calibrated-speedup} at the same runs' measured acceptance.
 Orange diamonds mark budgets excluded from fitting. Blue stars show separate
 RL observations, discussed in Appendix~\ref{app:why-ssrl-calibration}.}
 \label{fig:why-ssrl-budget-overlay}
\end{figure}

\Needspace{24\baselineskip}
\subsection{Interpreting the budget--acceptance speedup surface}
\label{app:why-ssrl-calibration}

Figure~\ref{fig:why-ssrl-surfaces} separates the budget sweep's two effects:
drafting cost and accepted output. At fixed $B$, $S$, and $k$, write
$B_{\mathrm{mix}}=\frac{2k+1}{k+1}B$ for the query batch in a mixed-stage
forward. Multiplying Eq.~\ref{eq:why-ssrl-mixed-latency} by $\tau$ and using
the linear-attention approximation
$T_{\mathrm{Attn}}(\lambda M)\simeq\lambda T_{\mathrm{Attn}}(M)$ gives
\[
 T_{\mathrm{mix}}\simeq
 (k+1)T_{\mathrm{GEMM}}(B_{\mathrm{mix}})
 +(1+ks)T_{\mathrm{Attn}}(M).
\]
Only the draft attention term depends on $s$ in this approximation.
We normalize by the AR step cost in the same component model,
$T_p=T_{\mathrm{GEMM}}(B)+T_{\mathrm{Attn}}(M)$, and add a dimensionless
constant $\delta$ fitted to account for execution and approximation errors.
Collecting the constant and budget-dependent terms yields
\[
 \frac{T_{\mathrm{mix}}}{T_p}+\delta\simeq
 \underbrace{\frac{(k+1)T_{\mathrm{GEMM}}(B_{\mathrm{mix}})
 +T_{\mathrm{Attn}}(M)}{T_p}+\delta}_{c_0}
 +s\underbrace{\frac{kT_{\mathrm{Attn}}(M)}{T_p}}_{c_1}.
\]
Since each speculative round produces $\tau=1+k\alpha$ tokens per request,
speedup is this output count divided by normalized round cost. Thus the
calibrated estimate is
\begin{equation}
 \widehat\eta(\alpha,s)=\frac{1+k\alpha}{c_0+c_1s},\qquad k=6.
 \label{eq:why-ssrl-calibrated-speedup}
\end{equation}
Here $c_0$ includes projection work, dense verification attention, and the
fitted correction; $c_1s$ is the draft attention contribution.
The ratio $s$ is the effective
fraction of historical KV tokens read by the drafter, accounting for the
minimum budget, rounding, and newly appended tokens.
% This is the original calibrated model in compact form:
% c_0 = (k+1) g_k (1-rho) + rho + delta; c_1 = k rho.
% The plotted surfaces and their fitted values have not been changed.

We profile component costs independently and fit only $\delta$ on the
budget sweep, holding out the 5\% and 20\% budgets for
validation. Acceptance enters as a measured input. Evaluating the expression
over $(s,\alpha)$ gives the background in Figure~\ref{fig:why-ssrl-surfaces},
with measurements overlaid. Figure~\ref{fig:why-ssrl-budget-overlay} applies
the same model at each median-speedup run's acceptance and effective budget.

\begin{figure}[h]
 \centering
 \newlength{\ssrlSurfacePanelWidth}
 \setlength{\ssrlSurfacePanelWidth}{0.43\linewidth}
 \begin{subfigure}[t]{\ssrlSurfacePanelWidth}
  \vspace{0pt}\centering
  \includegraphics[width=\linewidth]{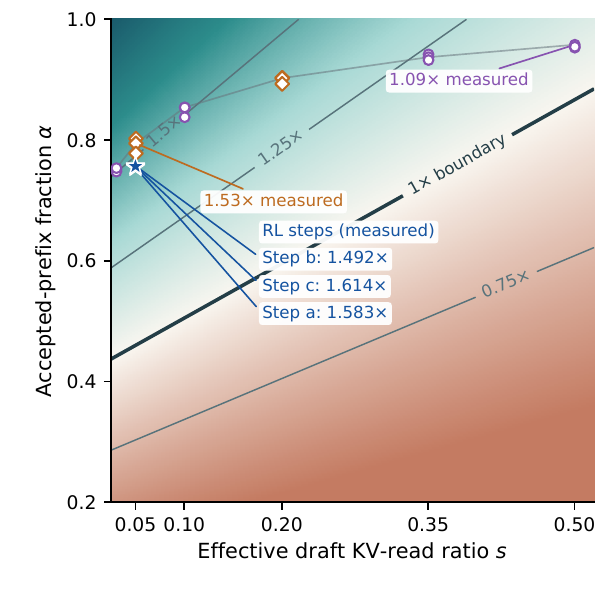}
  \caption{Qwen3-1.7B.}\label{fig:why-ssrl-surface-1-7b}
 \end{subfigure}\hfill
 \begin{subfigure}[t]{\ssrlSurfacePanelWidth}
  \vspace{0pt}\centering
  \includegraphics[width=\linewidth]{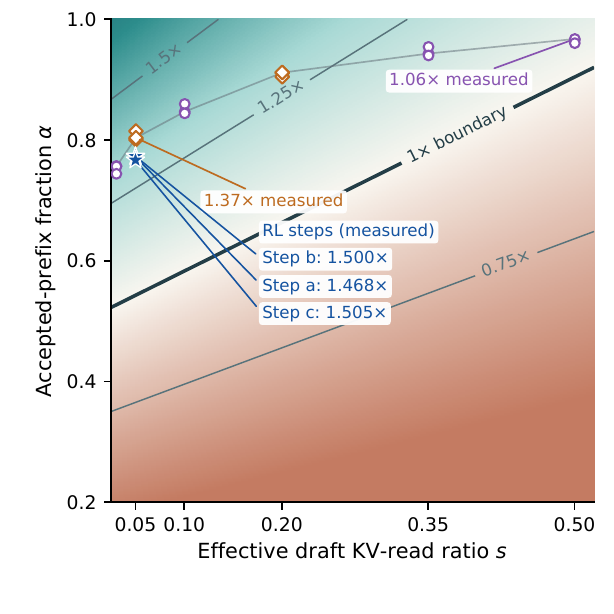}
  \caption{Qwen3-8B.}\label{fig:why-ssrl-surface-8b}
 \end{subfigure}\hfill
 \begin{minipage}[t]{0.10\linewidth}
  \vspace{0pt}\centering
  \includegraphics[height=0.83\ssrlSurfacePanelWidth]{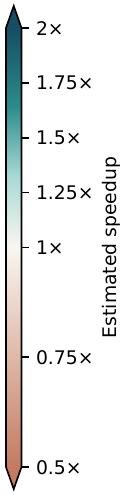}
 \end{minipage}
 \par\medskip
 \includegraphics[width=0.90\linewidth]{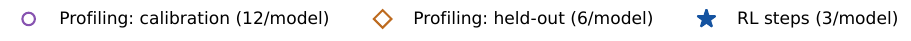}
 \caption{Estimated speedup over sparse budget and acceptance for
 (\subref{fig:why-ssrl-surface-1-7b}) Qwen3-1.7B and
 (\subref{fig:why-ssrl-surface-8b}) Qwen3-8B at $B=64$, $S\approx8$K, and $k=6$.
 The background evaluates Eq.~\ref{eq:why-ssrl-calibrated-speedup}; the dark
 contour marks $1\times$ speedup. Circles and diamonds show calibration and
 held-out observations. Blue stars locate separate RL observations at $s=0.05$;
 their annotations report measured speedups.}
 \label{fig:why-ssrl-surfaces}
\end{figure}

\Needspace{10\baselineskip}
\paragraph{Reading the surface.}
Moving upward at fixed budget increases accepted output and speedup;
moving rightward at fixed acceptance increases cost. Along the measured
trajectory, larger budgets improve acceptance, but their added cost
eventually dominates. The dark contour marks where useful progress just
covers round cost. All measured sweep points exceed $1\times$, so the
break-even boundary and low-acceptance region remain model extrapolations.
For \modelname's reuse of verification attention scores, the surface
suggests a concrete objective: improve accepted progress at the same
KV-read budget while keeping selection overhead small. Increasing context
coverage alone need not move the system toward higher throughput.

\paragraph{Calibration and execution cost.}
The constant correction reduces held-out mean absolute percentage error
from 84.61\% to 0.80\% for 1.7B and from 63.03\% to 0.68\% for 8B.
This gap motivates measuring and optimizing verification, sparse selection,
and scheduling alongside KV reads. It includes error in the component
approximation, so the correction cannot identify these contributions
individually. The small residual errors apply only to the profiled
execution settings.

\paragraph{Connection to RL rollouts.}
\label{app:why-ssrl-rl}
Blue stars show three RL steps per model at $s=0.05$. Pooling output counts
and generation times across all four paired 1.7B steps and three 8B steps
gives throughput speedups of $1.558\times$ and $1.491\times$, at acceptance
75.7\% and 77.1\%. These runs exclude optimization time and use changing
batches, CPU KV offloading, and a vLLM AR baseline, so local fitting accuracy
does not establish prediction accuracy for RL training. The complementary
gain from KV-aware scheduling in Section~\ref{sec:kv-scheduling-ablation}
points to a further target: coordinate prompt sharing and prefetching with
complete-group generation, so KV transfers do not consume the attention
savings.

\end{document}

\appendix
% Manuscript fragment: input after \appendix, with graphicx and natbib loaded.
% The standalone wrapper overrides the figure path before reading this file.
\providecommand{\modelname}{S3R}
\providecommand{\cotrainfigpath}{figs/co_training/}

\section{Checkpoint Portability and Online Drafter Adaptation}
\label{app:checkpoint-portability}

An RL workflow may start from an independently post-trained checkpoint for
which a matched auxiliary proposer is unavailable. We examine two practical
aspects of this setting: applying speculative decoding to different target
checkpoints without additional drafter training, and maintaining an auxiliary
proposer as the target policy evolves. These experiments complement the main
training results by evaluating the flexibility of sparse self-drafting.
Figure~\ref{fig:checkpoint-portability} compares checkpoint portability
and periodic online drafter adaptation.

% Three panels with one shared caption; requires graphicx and subcaption.
% Vector panels are prepared by figs/prepare_caption_panels.py.
\begin{figure}[htbp]
  \centering
  \includegraphics[width=0.65\linewidth]{\cotrainfigpath ../caption_panels/checkpoint_legend.pdf}\hfill\mbox{}
  \par\smallskip
  \newlength{\cotrainPanelHeight}
  \setlength{\cotrainPanelHeight}{0.247\linewidth}
  \begin{subfigure}[t]{0.32\linewidth}
    \centering
    \includegraphics[height=\cotrainPanelHeight]{\cotrainfigpath ../caption_panels/checkpoint_acceptance.pdf}
    \caption{Checkpoint acceptance.}\label{fig:checkpoint-acceptance}
  \end{subfigure}\hfill
  \begin{subfigure}[t]{0.32\linewidth}
    \centering
    \includegraphics[height=\cotrainPanelHeight]{\cotrainfigpath ../caption_panels/checkpoint_speedup.pdf}
    \caption{Checkpoint speedup.}\label{fig:checkpoint-speedup}
  \end{subfigure}\hfill
  \begin{subfigure}[t]{0.32\linewidth}
    \centering
    \includegraphics[height=\cotrainPanelHeight]{\cotrainfigpath ../caption_panels/online_adaptation.pdf}
    \caption{Online drafter adaptation.}\label{fig:online-drafter-adaptation}
  \end{subfigure}
  \caption{Checkpoint portability and periodic online drafter adaptation.
  (\subref{fig:checkpoint-acceptance}) Draft-token acceptance and
  (\subref{fig:checkpoint-speedup}) generation speedup relative to paired AR across target checkpoints;
  error bars summarize placement variability checked against additional runs.
  (\subref{fig:online-drafter-adaptation}) Mean acceptance length during periodic co-training,
  with faint observations and bold seven-step rolling means.}
  \label{fig:checkpoint-portability}
\end{figure}

\subsection{Checkpoint Portability without Auxiliary Proposer Training}
\label{app:checkpoint-evaluation}

\paragraph{Targets and protocol.}
We evaluate five Qwen3-8B-family targets: steps 30, 60, and 90 from a
DAPO--History Retrieval run; step 130 from a separate DAPO--\modelname{} run; and the
released Nemotron-Cascade-8B model~\citep{nemotronCascade2025}. The latter has
undergone both supervised fine-tuning and reinforcement learning. These
targets are treated as separate categories: step 130 and Nemotron do not
extend the training trajectory of steps 30--90.

All targets use the same 30 AIME24 problems, with 16 responses per problem,
temperature 0.6, engine seed 330, tensor parallelism of two, and a saturated
queue of 480 requests. Prompts use each target's tokenizer and chat template.
The DAPO checkpoints use a 40,960-token context limit and at most 32,768
generated tokens; Nemotron uses a 32,768-token context limit and
at most 32,336 generated tokens. This checkpoint study uses an eight-token
self-drafting stride for \modelname{} (labeled Ours in the figure), a
sparse budget ratio of 0.05, and a minimum budget of 192 tokens. Auxiliary Proposer uses
the fixed \texttt{AngelSlim/Qwen3-8B\_eagle3} drafter. Auxiliary Proposer and N-gram Lookup
propose five or eight draft tokens, while History Retrieval uses adaptive suffix-cache
drafting with a maximum of 16 tokens. All methods are evaluated without
additional drafter fitting or online drafter updates.
For this checkpoint study, History Retrieval's prebuilt checkpoint-0 trajectory cache
is held fixed during target evaluation.

AR, Auxiliary Proposer, and N-gram Lookup use vLLM 0.12.0. History Retrieval uses vLLM 0.10.0
with a separate AR control. We report the median throughput ratio to the
paired AR control from the same placement and runtime comparison group.
These are full-system generation speedups, including runtime differences;
they exclude RL optimization time and are distinct from the main paper's
end-to-end training speedups.
Draft-token acceptance measures the fraction of proposed tokens accepted;
its denominator depends on the draft configuration, so acceptance and
speedup should be interpreted together.
The direction glyphs below each panel summarize each method's relative
change from DAPO step 30 to Nemotron, computed as
$(v_{\mathrm{Nemotron}}/v_{\mathrm{DAPO30}}-1)\times100\%$ for the
corresponding metric $v$. These annotations compare separate target
categories rather than a continuous training trajectory; they do not
propagate uncertainty from the source measurements.

\paragraph{Placement variability.}
Speedup error bars summarize placement variability and were checked against
additional experimental observations. They are not confidence intervals
and do not represent variation across independent training seeds.
% The authors confirmed the follow-up variability check on 2026-09-17.

\paragraph{Results.}
Across the five targets, \modelname{} retains 76.32--78.78\% draft-token
acceptance and 1.698--1.799$\times$ generation speedup
(Figure~\ref{fig:checkpoint-portability}). Within the shared steps-30--90
trajectory, acceptance decreases modestly from 78.78\% to 77.76\%, while
speedup remains at 1.698--1.731$\times$. On Nemotron, \modelname{} achieves
76.74\% acceptance and 1.740$\times$ speedup without fitting a new drafter.
Fixed Auxiliary Proposer remains usable on this target: five drafts yield
1.136$\times$ speedup, whereas eight drafts yield 0.938$\times$.
Thus, direct reuse can provide reduced and configuration-sensitive
benefits. The result supports the practical value of obtaining proposals
directly from a selected supported checkpoint; it does not establish that
retraining an auxiliary proposer is necessary for every target.

\subsection{Periodic Online Drafter Adaptation}
\label{app:online-drafter-adaptation}

\paragraph{Online drafter setting and training.}
We analyze three historical Qwen3-8B GRPO runs on DAPO Math in our
VERL--SpeCo implementation, using EAGLE3, DFlash~\citep{dflash2026},
and the DSpark adaptation recipe~\citep{dspark2026}. Each run uses one
node with 16 NPU workers, five rollouts per prompt, and prompt/response
limits of 512/8,192 tokens. Target policies are optimized with AdamW at
a learning rate of $10^{-6}$. Target hidden states are collected through
the old-logprob forward pass, and the auxiliary drafter is trained every
five RL steps for 20 optimizer steps per update, followed by synchronous
weight publication. EAGLE3 uses a prompt batch size of 16, TP$=2$, and
three draft tokens; DFlash uses a batch size of 64, TP$=2$, and 15 draft
tokens; DSpark uses a batch size of 64, TP$=4$, and seven draft tokens.
EAGLE3 enables an actor KL penalty, while the other two runs do not.
Thus, drafter optimization is periodic rather than accompanying every
target-policy update.

\paragraph{Acceptance metric and visualization.}
Figure~\ref{fig:online-drafter-adaptation} shows the logged mean acceptance
length (\texttt{mean\_acceptance\_length}) for each run. We use this metric
to describe acceptance changes within each run; it is distinct from the
draft-token acceptance fraction used in the checkpoint experiment.
Centered seven-step rolling means are computed using the available
observations at each boundary. The curves retain their exported ranges:
DSpark through step 82, DFlash through step 96, and EAGLE3 through
step 104; these endpoints indicate only the extent of the available
records. The training configuration specifies the periodic drafter-update
schedule, and the local records identify drafter update and publication
events. However, these events are not isolated interventions: target-policy
evolution and rollout samples also change over training. We therefore do
not interpret acceptance changes around individual update events as their
causal effects.

\paragraph{Observations and implications.}
In these configurations, periodic adaptation does not produce a monotonic
increase in proposal acceptance. Comparing the same windows, steps 61--75
and 76--82, mean acceptance length decreases from 3.603 to 3.245 for
DSpark and from 3.761 to 2.607 for DFlash, while EAGLE3 remains
approximately stable, changing from 2.075 to 2.124. These windows summarize
the recorded curves rather than identify a causal change point. The
single-run observations describe within-run proposal-acceptance changes;
they neither establish instability of target-policy optimization nor
constitute a controlled ranking of the three methods. In particular,
although drafter update and publication events are recorded locally, the
available runs do not isolate their effects from concurrent target-policy
evolution and changing rollout samples.

Successful online co-training has been reported in
NeMo-RL~\citep{onlineDraftCoTraining2026}, which supports EAGLE3, DFlash,
and DSpark through branch-aware context parallelism and cross-stage
target-feature transport, accelerating RL while closely tracking the baseline
learning trajectory. The present case concerns the periodic adaptation
configuration used here. Maintaining acceptance can
require coordinating auxiliary proposer adaptation with policy evolution.
\modelname{} removes the need to select and initialize a separate drafter
and schedule its optimization: its sparse proposals and dense verification
share the synchronized rollout-policy snapshot. Acceptance can still vary
with the sparse approximation and the target distribution, but this design
avoids a separate source of parameter lag and auxiliary training effort.

\bibliographystyle{plainnat}
\bibliography{references}